\documentclass{article}

\usepackage[main,final]{neurips_2026}

\usepackage[utf8]{inputenc} 
\usepackage[T1]{fontenc} 
\usepackage[ colorlinks=true,linkcolor=blue,citecolor=blue,urlcolor=blue ]{hyperref}
\usepackage{url} 
\usepackage{booktabs} 
\usepackage{array}
\usepackage{amsfonts} 
\usepackage{nicefrac} 
\usepackage{microtype} 
\usepackage{xcolor} 
\usepackage{amsthm}
\usepackage{amsmath}
\usepackage{mathtools}
\usepackage{extarrows}
\usepackage{mathbbol}
\usepackage{float} 
\usepackage[nameinlink]{cleveref}

\crefname{section}{\S}{\S}
\crefname{subsection}{\S}{\S}

\usepackage{standalone}
\usepackage{graphicx}
\graphicspath{{figures/}}
\usepackage{wrapfig}
\usepackage{caption}
\usepackage{booktabs}
\usepackage{enumitem}
\setlist[enumerate]{itemsep=0.2ex,parsep=0pt,topsep=0pt,partopsep=0pt}
\setlist[itemize]{itemsep=0.2ex,parsep=0pt,topsep=0pt,partopsep=0pt}

\usepackage{tikz}
\usetikzlibrary{arrows}
\usetikzlibrary{arrows.meta}
\usetikzlibrary{backgrounds, fit, decorations.pathmorphing}
\usepackage{tikzscale}
\usepackage{subcaption}
\usepackage{algorithm}
\usepackage{algpseudocode}

\makeatletter
\let\c@example\c@theorem
\makeatother

\newtheorem{theorem}{Theorem}
\newtheorem{definition}{Definition}
\newtheorem{corollary}{Corollary}
\newtheorem{proposition}{Proposition}
\newtheorem{example}{Example}
\newtheorem{lemma}{Lemma}
\newtheorem{assumption}{Assumption}
\newtheorem{remark}{Remark}

\usepackage{xparse}

\title{Coarsening Linear Non-Gaussian Causal Models with Cycles}

\author{%
  Francisco Madaleno\textsuperscript{1}\textsuperscript\textdagger \quad\quad\quad
  Francisco C.~Pereira\textsuperscript{1} \quad\quad\quad
  Alex Markham\textsuperscript{2} \\
  \textsuperscript{1}Technical University of Denmark \\
  \textsuperscript{2}University of Copenhagen \\
  \textsuperscript\textdagger\texttt{fmfsa@dtu.dk}
}

\begin{document}

\maketitle

\begin{abstract}
  Recent work on causal abstraction, in particular graphical approaches focusing on causal structure between clusters of variables, aims to summarize a high-dimensional causal structure in terms of a low-dimensional one.
  Existing methods for learning such summaries from data assume that both structures are acyclic.
  Acyclicity is helpful for causal effect identification and reasoning, but it excludes the many systems of scientific interest that contain feedback loops.
  We show that in the linear non-Gaussian (LiNG) setting, the acyclicity assumption on the high-dimensional structure can be relaxed while still allowing recovery of a low-dimensional causal directed acyclic graph (DAG).
  LiNG models with cycles are observationally identifiable only up to an equivalence class whose members differ only within their feedback loops.
  Our low-dimensional DAG, which is invariant across all members of this class, is therefore identifiable from observational data alone.
  While existing approaches for learning this equivalence class have exponential time complexity, our low-dimensional summary is learned in worst-case cubic time and comes with finite-sample guarantees.
  We provide open source code and experiments on synthetic data to corroborate our theoretical results.
\end{abstract}

\section{Introduction}

Causal discovery from observational data has predominantly focused on directed \emph{acyclic} graphs (DAGs) \citep{shimizu2006lingam, shimizu2011directlingam}.
Yet many systems of scientific interest contain feedback loops: protein-signaling networks \citep{sachs2005causal} and dynamical systems \citep{mooij2013ode} exhibit directed cycles.
A line of work on \emph{cyclic} causal models has therefore developed, ranging from the structural-equation formulation of~\cite{spirtes1995directed,bongers2021foundations} to the linear non-Gaussian models of~\cite{lacerda2008discovering}, which leverage standard ICA \citep{hyvarinen2001independent, eriksson2004identifiability} to recover the demixing matrix $W = I - B$ of a linear cyclic Structural Causal Model (SCM) up to row permutation and scaling, from which the weighted adjacency matrix $B$---and hence the directed graph $G$---can be recovered up to a distributional equivalence class.
Additionally, early extensions of linear non-Gaussian methods estimate causal orderings among multidimensional, disjoint groups of variables \citep{kawahara2010grouplingam, entner2012estimating}.
The GroupLiNGAM algorithm of~\cite{kawahara2010grouplingam} has worst-case complexity exponential in the number of variables; \Cref{app:grouplingam} compares it with our approach.
More recently, this perspective has been revisited for models containing disjoint cycles \citep{pmlr-v286-drton25a}, with identification results that rely on second- and third-order moments and thus require the noise distribution to be skewed.
In contrast, our approach relies only on general non-Gaussianity via standard ICA, does not require disjoint cycles, and runs in polynomial time.
We compare empirically with that line of work, although its identification target---the full variable-level graph, assuming disjoint cycles---differs from ours.

The coarsening framework of~\cite{wahl2024foundations, madaleno2026coarsening} characterizes when a causal graph can be simplified by grouping variables into clusters.
We extend the coarsening lattice of~\citet{madaleno2026coarsening} to directed graphs with cycles: the strongly connected components (SCC) partition emerges as a structural floor on this lattice---the finest partition whose quotient is a DAG---and the resulting cluster-DAG is exactly the \emph{condensation} $G^{\mathrm{sc}}$ of $G$ in the classical sense of graph theory \citep{Bondy1976}.
This is the identification target that resolves the ambiguity in the equivalence class of~\cite{lacerda2008discovering}.

\paragraph{Contributions}

We extend the coarsening framework to directed graphs that may contain cycles, and connect it to the cyclic-SCM identification literature:
\begin{itemize}
  \item \textbf{The SCC floor (\cref{sec:scc-floor}).}
        The coarsening lattice acquires a structural floor: any DAG-coarsening must group nodes in a common SCC, and the SCC partition is the finest one.
  \item \textbf{Condensation identifiability (\Cref{thm:condensation-id}, \cref{sec:approach-oica}).}
        The condensation $G^{\mathrm{sc}}$ of a linear non-Gaussian cyclic SCM is invariant across the distributional equivalence class of~\cite{lacerda2008discovering}, and is therefore identifiable from observational data alone, without invoking the stability and disjoint-cycles assumptions of that work.
\end{itemize}
Finally, as a downstream use case, the identification of $G^{\mathrm{sc}}$ from observational data enables the C-DAG identification theorems of~\cite{anand2023causal}, reducing the identification of causal-effects at the cluster-level in a cyclic system to the standard acyclic case.

The paper is organized as follows: \cref{sec:background} introduces notation and the coarsening lattice and cyclic-SCM background; \cref{sec:cycles} introduces DAG-coarsenings of cyclic graphs, presents the learning procedure, and proves the SCC-floor proposition and condensation identifiability; \cref{sec:experiments} reports the synthetic experiments\footnote{An open source implementation as well as scripts for reproducing all of the following experiments can be found at \url{https://github.com/fmfsa/coarsening-cycles}.\label{fn:code}}; \cref{sec:discussion} concludes.

\section{Related Work}
\label{sec:related}

\textbf{Cyclic generalizations of structural causal models} date back to~\cite{spirtes1995directed}, with rigorous foundations in~\cite{bongers2021foundations}.~\cite{lacerda2008discovering}
showed that linear non-Gaussian cyclic SCMs with no latent confounders can be identified via standard ICA: ICA recovers the demixing matrix $\widehat{W}$, from which one obtains a candidate adjacency matrix $\widehat{B} = I - \mathrm{diag}(P\widehat{W})^{-1} P\widehat{W}$ for each admissible row permutation $P$, and thresholding yields the graph.
Later work has extended this to latent-confounded settings via overcomplete ICA \citep{salehkaleybar2020learning, dai2026distributional}, extensions of LiNGAM \citep{maeda2020rcd}, and constraint-based methods using $\sigma$-separation \citep{forre2018constraint}.~\cite{amendola2020structure}
study structure learning for linear cyclic SCMs including the Gaussian case, while~\cite{tramontano2022learning} adapt linear non-Gaussian models to learn graphs that are polytrees and~\cite{rothenhausler2015backshift} leverages interventional data.
The condensation has appeared as an identification target before: under our assumptions, the finest block partition of GroupLiNGAM \citep{kawahara2010grouplingam} coincides with the SCC partition, so it recovers the same object, though it relies on an additional faithfulness assumption and a worst-case exponential subset search that our identification avoids (\Cref{app:grouplingam}).
\citet{sharifian2025} prove this invariance for linear non-Gaussian cyclic models and leverage it for experiment design.
Our contribution is to place the SCC partition in the coarsening lattice as its structural floor, with a polynomial estimator and finite-sample guarantees.

Another line of work grounds causal models in \textbf{stochastic differential equations} \citep{HansenSokol2014}, rather than standard structural equations.
Graphical Lyapunov models \citep{pmlr-v124-varando20a,pmlr-v236-dettling24a,recke2026continuous,recke2026discrete} encode directed causal structure via a weighted adjacency matrix that solves either the continuous- or discrete-time Lyapunov equation, while naturally allowing for cycles;~\cite{pmlr-v236-rohbeck24a} estimate such models from interventional single-cell data.

\textbf{Coarsening of causal models} has been approached from several angles: causal abstraction in the SCM sense \citep{rubenstein2017causal,beckers2020approximate}, group-level structure under partial-order assumptions \citep{wahl2024foundations} and the partition-refinement approach of~\cite{madaleno2026coarsening} that we extend here.
On the abstraction side, \citet{massidda2024learning} learn a linear abstraction map from a high-dimensional linear SCM to a low-dimensional one, but assume both models acyclic and require paired data sampled jointly from the two models; our variable-level model is cyclic, the cycles themselves determine the clusters, and we observe high-dimensional data only.
\citet{schooltink2025aligning} align graphical abstractions such as C-DAGs with functional abstractions for semi-Markovian SCMs, that is, models without cyclic causal relations; our cyclic setting falls outside their assumptions, and we suggest as future work the extension of their results to our cyclic setting.
\Cref{app:related-work-table} compares these and further cyclic-model and cluster-level works along model class, confounding, data requirements, and clustering assumptions.
The C-DAG identification theorems of~\cite{anand2023causal} characterize which cluster-level interventional distributions are recoverable when an acyclic cluster-level graph is given; the condensation we identify in this paper is precisely such a graph, supplied from observational data.

The condensation construction is standard in graph theory, and \cite{tarjan1972depth}
gives a linear-time algorithm.
Our use of the SCC as the finest DAG-coarsening is, to our knowledge, a new approach to identifying cluster-level causal effects in cyclic models: we transform a cyclic model into an acyclic surrogate to which the standard identification machinery for DAGs applies.

\section{Background}
\label{sec:background}

\subsection{Coarsening causal models}
\label{sec:coarsening-dag-models}

Let $G^* = (V^*, E^*)$ be a DAG.
A \emph{coarsening} \citep{madaleno2026coarsening} of $G^*$ is a partition $\Pi = \{\pi_1,\ldots,\pi_k\}$ of $V^*$ such that the quotient graph $G' = (\Pi, E')$ with edges
\[
  E' = \{(\chi(u^*), \chi(v^*)) \mid (u^*,v^*) \in E^*,\ \chi(u^*) \neq \chi(v^*)\}
\]
is itself a DAG, where $\chi: V^* \to \Pi$ sends each node to its part.
We call $G'$ the \emph{quotient graph} and the parts $\pi_i$ \emph{clusters}.

Let $\Pi$ and $\Pi'$ be partitions of a set $V$.
We write $\Pi \preceq \Pi'$ when every part of $\Pi$ is contained in a part of $\Pi'$; thus $\Pi$ is \emph{finer} than $\Pi'$, and $\Pi'$ \emph{coarser} than $\Pi$.
The set of all partitions of $V$ under this refinement order is the \emph{partition refinement lattice}: the \emph{meet} $\Pi \wedge \Pi'$ of two partitions is their coarsest common refinement, and the \emph{join} $\Pi \vee \Pi'$ is their finest common coarsening.
The coarsenings of $G^*$ themselves form a lattice under refinement \citep[Theorem~3]{madaleno2026coarsening}.
More precisely, they form a \emph{meet-subsemilattice} of the partition refinement lattice \citep[Proposition~24]{tedescoScoreBasedLearningCluster2026a}: their meet is inherited from the partition lattice, while their join within the coarsening lattice is the finest \emph{valid} common coarsening and may differ from the join in the full partition lattice.
We generalize this construction to cyclic graphs in \cref{sec:cycles}.

\subsection{Cyclic structural causal models}
\label{sec:cyclic-scm}

A linear cyclic SCM is defined by:
\begin{equation}
  \label{eq:sem}
  X = BX + \varepsilon, \qquad \text{solved as } X = (I - B)^{-1} \varepsilon,
\end{equation}
where $d = |V|$ is the number of variables, $B \in \mathbb{R}^{d \times d}$ is the weighted adjacency matrix of a directed graph $G = (V, E)$ (which may contain cycles), and the noise vector $\varepsilon \in \mathbb{R}^d$ has mutually independent components.
The spectral radius $\rho(B) = \max_i |\lambda_i(B)|$ controls when \eqref{eq:sem} is solvable: existence and uniqueness only require $\det(I - B) \neq 0$, while the stronger condition $\rho(B) < 1$ additionally ensures dynamical stability and is the regime studied by~\cite{lacerda2008discovering}.

Two matrices represent important concepts: the \emph{mixing matrix} $A = (I - B)^{-1}$ maps sources to observations ($X = A\varepsilon$), with its support encoding the transitive closure of $G$ under no exact cancellations; the \emph{demixing matrix} $W = I - B = A^{-1}$ recovers sources ($\varepsilon = WX$), with its off-diagonal support encoding the direct adjacency of $G$.
Throughout, $\widehat W$ denotes the sample-level FastICA estimate of $W$.

\section{Coarsening Directed Graphs with Cycles}
\label{sec:cycles}


\begin{definition}[DAG-coarsening of a directed graph]
  \label{def:dag-coarsening}
  Let $G = (V, E)$ be a directed graph (possibly containing cycles), and let $\Pi = \{\pi_1, \ldots, \pi_k\}$ be a partition of $V$ with associated surjection $\chi : V \to \Pi$.
  The \emph{quotient graph} is $G' = (\Pi, E')$ with
  \[
    E' = \{(\chi(u), \chi(v)) \mid (u, v) \in E,\ \chi(u) \neq \chi(v)\}.
  \]
  We call $\Pi$ a \emph{DAG-coarsening} of $G$ when $G'$ is acyclic.
\end{definition}

When $G$ is a DAG, this reduces to the original definition of~\cite{madaleno2026coarsening} (the singleton-per-node partition is trivially valid).
For graphs containing cycles, the acyclicity requirement is non-trivial and forces SCCs to be respected, as \Cref{prop:scc-floor} below shows.

\subsection{The SCC floor}
\label{sec:scc-floor}

\begin{definition}[SCC coarsening/condensation]
  \label{def:scc-coarsening}
  Given a directed graph $G = (V, E)$, the \emph{SCC coarsening} of $G$ is the DAG-coarsening whose partition surjection $\chi$ satisfies
  \[
    \text{for all } v, w \in V, \quad \chi(v) = \chi(w) \iff \mathrm{sc}_G(v) = \mathrm{sc}_G(w),
  \]
  where $\mathrm{sc}_G(v)$ is the strongly connected component of $G$ containing $v$.
  The induced quotient is the \emph{condensation} of $G$ \citep{Bondy1976}, denoted $G^{\mathrm{sc}}$.
  The underlying partition coincides with the \emph{maximally acyclic} partition of~\cite{wahl2024foundations} (their Definition~8(ii)).
\end{definition}

Under the refinement order of \cref{sec:coarsening-dag-models}, a DAG-coarsening $\Pi_\star$ of $G$ is the \emph{finest DAG-coarsening} if $\Pi_\star \preceq \Pi$ for every DAG-coarsening $\Pi$ of $G$.
Equivalently, $\Pi_\star$ is the least element of the poset of DAG-coarsenings, which we call the \emph{structural floor}.

\begin{proposition}[SCC floor]
  \label{prop:scc-floor}
  Let $G = (V, E)$ be a directed graph.
  Nodes $u, v \in V$ in the same SCC of $G$ must lie in the same cluster of every DAG-coarsening of $G$.
  Consequently, the SCC coarsening of \Cref{def:scc-coarsening} is the finest DAG-coarsening of $G$---the structural floor of the DAG-coarsening lattice.
\end{proposition}

\begin{proof}[Proof sketch]
  The condensation $G^{\mathrm{sc}}$ is acyclic---a directed cycle through distinct SCCs would make their union strongly connected, contradicting the maximality of SCCs---so the SCC coarsening is a DAG-coarsening.
  Every DAG-coarsening groups whole SCCs, since splitting one would create a cycle in the quotient; hence the SCC coarsening refines every DAG-coarsening and is the least element, and the DAG-coarsenings moreover form a lattice under refinement.
  The complete proof is in \Cref{app:scc-floor-proof}.
\end{proof}

Given $G$ the SCC partition is computable via Tarjan's algorithm~\citep{tarjan1972depth}.

\begin{remark}[cyclic coarsenings]
  \Cref{def:dag-coarsening,def:scc-coarsening} require the coarsened graph to be a DAG, mirroring the original framework of~\cite{madaleno2026coarsening}.
  A natural generalization drops this requirement and asks only that the coarsening preserves the partial order over SCCs of $G$: any partition $\chi$ such that the quotient graph (a) does not reverse a non-cyclic ancestral relation in $G$, and (b) does not introduce a cycle between nodes that are not already in a common SCC, is a valid \emph{cyclic coarsening}.
  The interventional and observational coarsenings considered below are all DAG-coarsenings (and hence cyclic coarsenings), so we focus on the DAG case throughout.
  Characterizing the general cyclic-coarsening lattice is left to future work.
\end{remark}

\begin{example}[example of~\cite{lacerda2008discovering}]
  \label{ex:illustrative}
  Consider the five-variable linear cyclic SCM
  \[
    X_1 = \varepsilon_1,\quad X_2 = 1.2X_1 - 0.3\,X_4 + \varepsilon_2,\quad X_3 = 2X_2 + \varepsilon_3,\quad X_4 = -X_3 + \varepsilon_4,\quad X_5 = 3X_2 + \varepsilon_5.
  \]
  The directed graph has edges $X_1\! \to \!X_2$ and $X_2\! \to \!X_3$, $X_3\! \to \!X_4$, $X_4\! \to \!X_2$ (a 3-cycle) and $X_2\! \to \!X_5$ (\Cref{fig:example} left).
  The cycle $X_2\! \to \!X_3\! \to \!X_4\! \to \!X_2$ forms the unique non-trivial SCC $\{X_2,X_3,X_4\}$ while nodes $X_1$ and $X_5$ are singletons.
  Since any DAG-coarsening must absorb each SCC into a single cluster (\Cref{prop:scc-floor}), the finest valid coarsening has three clusters $\pi_1 = \{X_1\}$, $\pi_2 = \{X_2,X_3,X_4\}$, $\pi_3 = \{X_5\}$, yielding the chain condensation $G^{\mathrm{sc}} = \pi_1 \to \pi_2 \to \pi_3$ in \Cref{fig:example}, right.
  This SCC partition is therefore the structural floor, and \Cref{fig:lattice} shows it as the least element of the DAG-coarsening lattice.
  Reversing the 3-cycle to $X_2 \to X_4 \to X_3 \to X_2$ (\Cref{fig:example}, center) yields a second variable-level graph with re-weighted edges that induces the same observational distribution \citep[Theorem~4]{lacerda2008discovering}.
  ICA cannot distinguish the two, but they share the condensation $G^{\mathrm{sc}}$, and this invariance is formalized in \Cref{thm:condensation-id}.
\end{example}

\begin{figure}[t]
  \centering
  \resizebox{\textwidth}{!}{%
    \begin{tikzpicture}[ vertex/.style={circle, draw, thick, minimum size=7mm, inner sep=0pt, font=\footnotesize\bfseries},intra/.style={-{Stealth[length=2mm]}, thick},cross/.style={-{Stealth[length=2mm]}, gray!60!black, densely dashed},sccbox/.style={rounded corners=6pt, dashed, thick, inner sep=5pt},macro/.style={circle, draw, thick, minimum size=7mm, font=\small\sffamily},edge/.style={-{Stealth[length=2.5mm]}, thick},panellabel/.style={font=\small\sffamily} ]

      \node[vertex] (x1) at (-0.5, 1.0) {$X_1$};
      \node[vertex] (x2) at (1.6, 1.6) {$X_2$};
      \node[vertex] (x3) at (2.4, 0.4) {$X_3$};
      \node[vertex] (x4) at (0.8, 0.4) {$X_4$};
      \node[vertex] (x5) at (4.0, 1.0) {$X_5$};

      \begin{scope}[on background layer]
        \node[ sccbox, draw=pink!60, fill=pink!8, fit=(x2)(x3)(x4), label={[font=\scriptsize\sffamily, pink!70!black]above:$\pi_2$} ] {};
      \end{scope}

      \draw[intra] (x2) -- (x3);
      \draw[intra] (x3) -- (x4);
      \draw[intra] (x4) -- (x2);
      \draw[cross] (x1) -- (x2);
      \draw[cross] (x2) -- (x5);

      \node[vertex] (y1) at (5, 1.0) {$X_1$};
      \node[vertex] (y4) at (7.1, 1.6) {$X_4$};
      \node[vertex] (y2) at (7.9, 0.4) {$X_2$};
      \node[vertex] (y3) at (6.3, 0.4) {$X_3$};
      \node[vertex] (y5) at (9.5, 1.0) {$X_5$};

      \begin{scope}[on background layer]
        \node[ sccbox, draw=pink!60, fill=pink!8, fit=(y2)(y3)(y4), label={[font=\scriptsize\sffamily, pink!70!black]above:$\pi_2$} ] {};
      \end{scope}

      \draw[intra] (y3) -- (y2);
      \draw[intra] (y4) -- (y3);
      \draw[intra] (y2) -- (y4);
      \draw[cross] (y1) -- (y4);
      \draw[cross] (y4) -- (y5);

      \node[macro] (S1) at (12.0, 2.1) {$\pi_1$};
      \node[macro, fill=pink!10, draw=pink!60] (S234) at (12.0, 1.0) {$\pi_2$};
      \node[macro] (S5) at (12.0, -0.1) {$\pi_3$};

      \draw[edge] (S1) -- (S234);
      \draw[edge] (S234) -- (S5);

    \end{tikzpicture}}
  \caption{
    Distributional equivalence and condensation for the five-variable cyclic SCM of~\cite{lacerda2008discovering} (\Cref{ex:illustrative}).
    (left, center) Two members of the distributional equivalence class recovered by ICA, differ only in the orientation of the 3-cycle inside the SCC $\pi_2 = \{X_2, X_3, X_4\}$; both correspond to admissible permutations of the demixing matrix $\widehat W$.
    (right) The shared condensation $G^{\mathrm{sc}}$ with clusters $\pi_1 = \{X_1\}$, $\pi_2 = \{X_2, X_3, X_4\}$, $\pi_3 = \{X_5\}$.
    ICA cannot distinguish (left) from (center), but $G^{\mathrm{sc}}$ is identifiable (\Cref{thm:condensation-id}).
    Note that the inter-SCC edge from $X_1$ lands on $X_2$ in (left) and on $X_4$ in (center); the cluster-level edge $\pi_1 \to \pi_2$ of $G^{\mathrm{sc}}$ is unchanged.
  }
  \label{fig:example}
\end{figure}
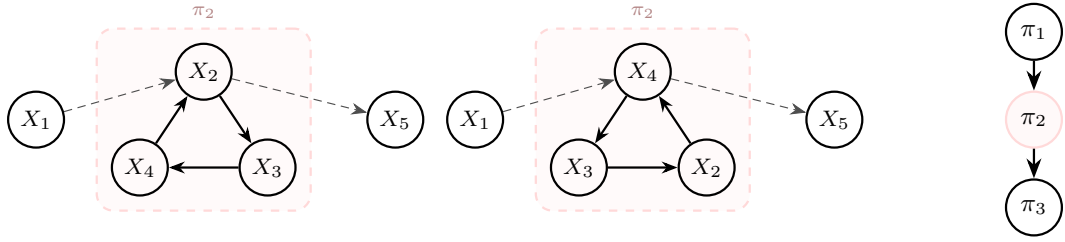

\begin{remark}[combinatorial structure]
  For $n = 5$ the partition lattice has $B_5 = \sum_{k=1}^{5} S(5,k) = 1 + 15 + 25 + 10 + 1 = 52$ partitions, where $S(n,k)$ is the Stirling number of the second kind.
  Only $4$ are valid DAG-coarsenings of \Cref{ex:illustrative}: the single-cluster partition $\{X_1X_2X_3X_4X_5\}$, the SCC floor $\{X_1\} \mid \{X_2X_3X_4\} \mid \{X_5\}$, and the two intermediate $2$-clusters $\{X_1\} \mid \{X_2X_3X_4X_5\}$ and $\{X_1X_2X_3X_4\} \mid \{X_5\}$.
  The remaining $48$ are invalid for two reasons (\Cref{fig:lattice}): they either split the SCC $\{X_2,X_3,X_4\}$, or they merge across condensation layers (e.g., $\{X_1,X_5\} \mid \{X_2,X_3,X_4\}$, where the edges $X_1\! \to \!X_2$ and $X_2\! \to \!X_5$ create a cycle in the quotient).
\end{remark}

\begin{figure}[t]
  \centering
  \resizebox{\textwidth}{!}{
    \begin{tikzpicture}[ node distance=1.5cm and 2.5cm,valid/.style={rectangle, rounded corners, fill=green!40!white, inner sep=4pt, font=\sffamily, draw=green!60!black},invalid/.style={rectangle, inner sep=4pt, font=\sffamily, text=gray},invalidbox/.style={rectangle, rounded corners, fill=red!10!white, inner sep=4pt, font=\sffamily, draw=red!40!white},dots/.style={font=\Huge, text=gray} ]

      \node[valid] (top) at (0, 0) {$X_1X_2X_3X_4X_5$};

      \node[valid] (v_left) at (-4.0, -1.5) {$X_1 \mid X_2X_3X_4X_5$};
      \node[invalidbox] (inv_mid) at ( 0.0, -1.5) {$X_1X_5 \mid X_2X_3X_4$};
      \node[valid] (v_right) at ( 4.0, -1.5) {$X_1X_2X_3X_4 \mid X_5$};
      \node[invalid] (other_2) at ( 8.5, -1.5)
      {$\dots$ 12 other invalid 2-clusters};

      \node[valid] (scc_floor) at (0, -3.5) {$X_1 \mid X_2X_3X_4 \mid X_5$};
      \node[invalid] (other_3) at (4.5, -3.5)
      {$\dots$ 24 other invalid 3-clusters};

      \node[dots] (d1) at (-1.4, -4.1) {$\vdots$};
      \node[dots] (d2) at ( 1.4, -4.1) {$\vdots$};
      \node[invalid] (other_4) at (4.0, -4.1)
      {$\dots$ 10 invalid 4-clusters};

      \node[invalidbox] (split_ex) at (0, -5) {$X_1 \mid X_2 \mid X_3 \mid X_4 \mid X_5$};

      \draw (top) -- (v_left);
      \draw (top) -- (inv_mid);
      \draw (top) -- (v_right);

      \draw (v_left) -- (scc_floor);
      \draw (inv_mid) -- (scc_floor);
      \draw (v_right) -- (scc_floor);

      \draw[gray, dashed] (scc_floor) -- (split_ex);

      \node[ anchor=east, font=\sffamily\bfseries\color{gray} ] at (-2.0, -3.5) {SCC Partition Floor};

    \end{tikzpicture}}
  \caption{
    The DAG-coarsening lattice (green) for the graph of \Cref{ex:illustrative}, embedded in the full partition lattice ($B_5 = 52$ partitions).
    Only 4 of the 52 partitions are valid; in some cases cycles are created, e.g the 2-cluster node $\{X_1,X_5\} \mid \{X_2,X_3,X_4\}$ is invalid because merging $X_1$ and $X_5$ creates a cycle between the two clusters; in other cases SCCs are split, e.g $\{X_1,X_2\} \mid \{X_3,X_4,X_5\}$ split the SCC $\{X_2,X_3,X_4\}$.
    The gray dashed edge and vertical ellipses mark the invalid region below the floor: every refinement of the floor splits an SCC.
  }
  \label{fig:lattice}
\end{figure}
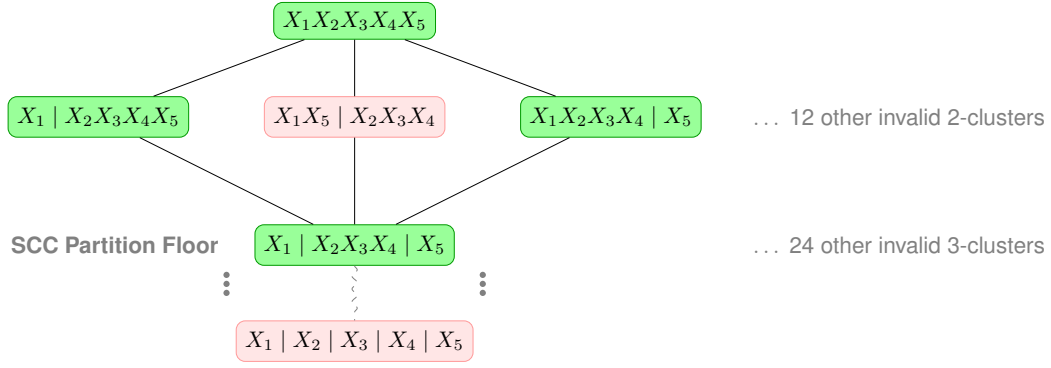

\subsection{Observational coarsening via linear non-Gaussian ICA}
\label{sec:approach-oica}

We assume the linear cyclic SCM of~\eqref{eq:sem} with mutually independent, non-Gaussian noise components $\varepsilon_i$ and \emph{no latent confounders}, so $\varepsilon$ has the same dimension as $X$.
Under these assumptions, standard ICA \citep{eriksson2004identifiability} recovers $W = I - B$ up to row permutation and scaling---the regime leveraged by~\cite{lacerda2008discovering}.
Overcomplete ICA \citep{lewicki2000learning} is needed only when the number of independent sources exceeds the number of observed variables, which arises under latent confounding \citep{salehkaleybar2020learning, dai2026distributional}.
The interventional case is discussed in \Cref{app:interventional}.

Let $\widehat{W}$ denote the ICA estimate.
Each row permutation $P$ such that $P\widehat{W}$ has nonzero diagonal entries yields an admissible candidate \citep{lacerda2008discovering}:
\begin{equation}
  \label{eq:B-from-W}
  \widehat{B}_P \; = \; I \; - \; \mathrm{diag}(P\widehat{W})^{-1} P\widehat{W}.
\end{equation}
The off-diagonal support of $\widehat{B}_P$, after thresholding at $\tau$, gives a candidate directed graph $\widehat{G}_P$.
Different admissible permutations $P$ yield different members of the distributional equivalence class.~\cite{lacerda2008discovering}
resolve this ambiguity by selecting the first candidate with $\rho(\widehat{B}_P) < 1$, the \emph{first-stable filter}, which under their stability assumption recovers the true $B$.

\begin{theorem}[Condensation identifiability]
  \label{thm:condensation-id}
  Let $G_1, G_2$ be two directed graphs underlying linear cyclic SCMs of the form~\eqref{eq:sem}, both with no latent confounders, mutually independent noise terms and at most one Gaussian noise component.
  If they induce the same observational distribution over $X$, then their condensations are equal: $G_1^{\mathrm{sc}} = G_2^{\mathrm{sc}}$ (same SCC partition, same set of inter-SCC edges).
\end{theorem}

\begin{proof}[Proof sketch]
  ICA identifiability reduces distributional equality to the choice of an admissible row permutation: each variable receives, up to scale, the structural equation of itself or of one of its parents (\Cref{lem:perm-condensation}).
  A counting argument on the condensation poset---purely graph-theoretic, with no structural assumption on $G$---shows that any admissible permutation maps every SCC onto itself and preserves the inter-SCC edges.
  Hence $G_1^{\mathrm{sc}} = G_2^{\mathrm{sc}}$; the complete proof is in \Cref{app:condensation-id-proof}.
\end{proof}

\begin{corollary}
  \label{cor:condensation-id}
  In the large-sample limit, $G^{\mathrm{sc}}$ is recoverable by running FastICA, picking one admissible row permutation via the Hungarian algorithm \citep{kuhn1955hungarian} on $-\log|\widehat W|$ (cells $\leq \eta$ forbidden, with $\eta$ fixed small; \Cref{app:sample-complexity}) to obtain a candidate $\widehat{B}$ via~\eqref{eq:B-from-W}, and applying Tarjan's algorithm to $\operatorname{supp}(\widehat{B})$.
  \Cref{thm:condensation-id} guarantees that the choice of admissible representative does not affect the recovered condensation.
\end{corollary}

A second consequence of SCC aggregation is the validity of conditional-independence at the cluster level.

\begin{lemma}[Markov on the condensation]
  \label{lem:cond-faith}
  Under the linear cyclic SCM of~\eqref{eq:sem}, the marginal joint distribution on cluster-level features is Markov to the condensation $G^{\mathrm{sc}}$.
\end{lemma}

Linear cyclic SCMs are $\sigma$-Markov to $G$ \citep{spirtes1995directed,bongers2021foundations}.
Aggregating variables to their SCCs projects $\sigma$-separation statements about $G$ onto separation statements about $G^{\mathrm{sc}}$; since $G^{\mathrm{sc}}$ is acyclic, $\sigma$-separation collapses to ordinary d-separation, so the cluster-level distribution is Markov to $G^{\mathrm{sc}}$ (proof in \Cref{app:cond-faith-proof}).
Furthermore, our identification does not invoke faithfulness: \Cref{alg:condensation} recovers $\widehat{G}^{\mathrm{sc}}$ by thresholding entries of $\widehat{B}$, not by conditional-independence testing.
Nonetheless, Corollary~2 of~\cite{wahl2024foundations} shows it is preserved under the SCC partition: $\sigma$-Markovian and $\sigma$-faithful $G$ implies the same for $G^{\mathrm{sc}}$.

\Cref{alg:condensation} packages this as a procedure requiring only observational data, based on ICA-LiNG-D \citep{lacerda2008discovering} composed with Tarjan, viewed through the coarsening lattice.
Finite-sample bounds in \Cref{app:sample-complexity}, in terms of \(\beta_{\min}\) (the smallest nonzero magnitude of the matched candidate), imply that \(n = \operatorname{poly}(d)\,(1 + \beta_{\min}^{-2})/\sqrt\delta\) samples suffice to recover the data-generating \(G^{\mathrm{sc}}\) with probability at least \(1 - \delta\), up to the constants of \Cref{assum:candidate-consistency}.

\begin{algorithm}[t]
  \caption{Condensation recovery from observational data.}
  \label{alg:condensation}
  \begin{algorithmic}[1]
    \Require Observations $X \in \mathbb{R}^{n \times d}$; thresholds $\tau, \eta > 0$
    \Ensure Condensation $\widehat{G}^{\mathrm{sc}} = (\widehat\Pi, \widehat E')$
    \State $\widehat{W} \gets \Call{FastICA}{X, d}$ \Comment{\citep{hyvarinen1999fast}, $d$ components}
    \State $P \gets \Call{Hungarian}{\widehat{W}}$ \Comment{maximize $\sum_i \log |(P\widehat{W})_{ii}|$ over entries $> \eta$; fails iff no such $P$ exists}
    \State $\widehat{B} \gets I - \operatorname{diag}(P\widehat{W})^{-1} P\widehat{W}$ \Comment{from \eqref{eq:B-from-W}}
    \State $\widehat{B}_{ij} \gets 0$ whenever $|\widehat{B}_{ij}| < \tau$ \Comment{thresholding}
    \State $\widehat\Pi \gets \Call{Tarjan}{\operatorname{supp}(\widehat{B})^\top}$ \Comment{SCCs \citep{tarjan1972depth}; transpose restores row$=$source}
    \State $\widehat{E}' \gets \{ (\pi', \pi) : \pi \neq \pi', \ \exists (i, j) \in \operatorname{supp}(\widehat{B}) \text{ with } i \in \pi,\ j \in \pi' \}$ \Comment{$B_{ij} \neq 0$ encodes $X_j \to X_i$}
    \State \Return $\widehat{G}^{\mathrm{sc}} = (\widehat\Pi, \widehat E')$

  \end{algorithmic}
\end{algorithm}

Enumerating all admissible permutations (the original recipe of~\cite{lacerda2008discovering}) is unnecessary for our identification target: by \Cref{thm:condensation-id}, $\widehat{G}^{\mathrm{sc}}$ does not depend on the choice of admissible permutation, and \Cref{cor:condensation-id} gives consistency of \Cref{alg:condensation} in the large-sample limit.
An empirical comparison (\Cref{app:ablation}) confirms that Hungarian selection matches or outperforms first-stable enumeration at a fraction of the cost, while non-admissible permutations never recover the condensation.

\begin{proposition}[Two failure modes of thresholding]
  \label{prop:conservative}
  Finite-sample noise affects the recovered partition in two opposite directions.
  (a) A true intra-SCC edge whose magnitude falls below $\tau$ is dropped; if this breaks reachability inside an SCC, the estimated partition \emph{splits} a true SCC into multiple clusters.
  (b) A spurious entry surviving above $\tau$ adds a false edge; if it closes a directed path between two distinct SCCs, the estimated partition \emph{merges} them.
  Mode (b) yields a coarsening of the true SCC partition; mode (a) yields a refinement and the resulting quotient is, in general, not a DAG-coarsening of the true graph.
\end{proposition}

We probe both failure modes empirically in \Cref{app:threshold-sensitivity} and discuss their connection to the sample complexity in \Cref{app:sample-complexity}.

\begin{theorem}[Time Complexity]
  \label{thm:time-complexity}
  Let $d = |V|$ be the number of observed variables and $n$ the sample size.
  \Cref{alg:condensation} recovers $G^{\mathrm{sc}}$ in worst-case time $\mathcal{O}(k \cdot n \cdot d^2 + d^3)$, where $k$ denotes the number of FastICA iterations to convergence.
\end{theorem}

The proof and a step-by-step breakdown are in \Cref{app:time-complexity}.
Because reliable ICA estimation requires $n \gg d$, the $\mathcal{O}(k \cdot n \cdot d^2)$ term dominates in practice.
By targeting $G^{\mathrm{sc}}$ rather than enumerating the full distributional equivalence class, we avoid the worst-case $\mathcal{O}(d!)$ N-rooks bottleneck.

\subsection{Causal Effect Identification on the Condensation}
\label{sec:effects}

\paragraph{Interventions in cyclic SCMs}

Under a soft (shift) intervention $X^{I} = (I - B)^{-1}(\varepsilon + \delta)$, effects propagate through the entire graph, including back through cycles.
Hard ($\mathrm{do}$) interventions that set a single variable $X_k = c$ inside a non-trivial SCC are excluded: they sever the equation for $X_k$, breaking the cyclic mechanism (potentially rendering the system inconsistent); even when well-defined, their effects depend on the unidentifiable intra-SCC structure.
However, hard interventions on every node of a cluster $\pi$ respect the SCC structure and yield well-defined interventional distributions on $G^{\mathrm{sc}}$:

Fix a cluster $\pi \subseteq V$ and let $\bar\pi = V \setminus \pi$.
Partition $X$, $\varepsilon$, and $B$ accordingly:
\[
  X = \begin{pmatrix}
        X_\pi       \\
        X_{\bar\pi}
      \end{pmatrix},
  \qquad \varepsilon =
  \begin{pmatrix}
    \varepsilon_\pi       \\
    \varepsilon_{\bar\pi}
  \end{pmatrix},
  \qquad B =
  \begin{pmatrix}
    B_{\pi\pi}     & B_{\pi\bar\pi}     \\
    B_{\bar\pi\pi} & B_{\bar\pi\bar\pi}
  \end{pmatrix}.
\]
A \emph{hard cluster intervention} $\mathrm{do}(X_\pi = c)$ replaces the equations indexed by $\pi$ with $X_\pi = c$.
Solving the remaining $|\bar\pi|$-system gives $X_{\bar\pi} = (I - B_{\bar\pi\bar\pi})^{-1}(B_{\bar\pi\pi}c + \varepsilon_{\bar\pi})$.
When $\pi$ is a union of SCCs, $B_{\bar\pi\bar\pi}$ inherits its spectrum from a subset of $B$, so $\det(I - B) \neq 0$ already implies invertibility.

A \emph{soft cluster intervention} adds a deterministic offset $\delta_\pi \in \mathbb{R}^{|\pi|}$ to the equations of $X_\pi$: $X_\pi = B_{\pi,V} X + \varepsilon_\pi + \delta_\pi$.
Soft interventions preserve the cyclic mechanism inside $\pi$ and induce a cluster-level shift on $G^{\mathrm{sc}}$.

\paragraph{C-DAG identification}

Once $G^{\mathrm{sc}}$ is in hand, cluster-level identification reduces to the standard acyclic case: for clusters $\pi_i, \pi_j$ of $G^{\mathrm{sc}}$, the interventional distribution $P(X_{\pi_j} \mid \mathrm{do}(X_{\pi_i} = c))$ is identifiable from $P(X)$ via the cluster-level $g$-formula---causal sufficiency leaves $G^{\mathrm{sc}}$ free of inter-cluster confounding (the C-DAG theorems of~\cite{anand2023causal} cover the confounded case).
We evaluate this end-to-end pipeline---recovered condensation, cluster-level $g$-formula propagation---on synthetic whole-SCC hard interventions in \Cref{app:intervention-expt}.

\section{Experimental Results}
\label{sec:experiments}

We illustrate \Cref{thm:condensation-id} on synthetic linear non-Gaussian cyclic SCMs using the recipe of~\cite{lacerda2008discovering} (\cref{sec:approach-oica}).
The interventional case (\Cref{app:interventional}) inherits the unmodified guarantees of~\cite{madaleno2026coarsening} and is not explored empirically.
All experiments ran on a single CPU workstation (AMD Ryzen 9 9950X, 128 GB RAM); fit times are reported in \Cref{fig:scalability}.

\subsection{Setup}

\paragraph{Data-generating process}

We generate linear non-Gaussian cyclic SCMs on $d = 10$ variables with a controlled number $\kappa \in \{3, 4, 5\}$ of non-trivial SCCs and edge densities $\lambda \in \{0.3, 0.5, 0.8\}$.
Each $(d, \kappa, \lambda)$ randomly partitions the $d$ nodes into $\kappa$ non-trivial SCCs (size $\geq 2$) plus singletons.
Within each SCC, a directed Hamilton cycle ensures mutual reachability; additional intra-SCC edges are sampled with probability $\lambda$.
Between clusters, edges respecting a random topological order are sampled with probability $\lambda$.
Edge weights are sampled in $[0.5, 0.95]$ with random sign.
Noise components $\varepsilon_i$ are mutually independent and Laplace-distributed.
We evaluate sample sizes $n \in [50, 10^5]$, with 10 random seeds per cell.

Each configuration is run under two spectral-radius regimes:
\textbf{Stable} ($\rho(B) \approx 0.9$): the first-stable filter can recover the true $B$.
\textbf{Unstable} ($\rho(B) = 1.5$): the first-stable filter does not necessarily select the right equivalence-class member; by \Cref{thm:condensation-id} the condensation is the same.

\paragraph{Method}

\Cref{alg:condensation} recovers $G^{\mathrm{sc}}$ from any single admissible permutation.
For the $d = 10$ experiments of \cref{sec:results} we instead run the full ICA-LiNG-D enumeration of~\cite{lacerda2008discovering}, which additionally yields variable-level $F_1$ and visualizes the gap between the identifiable cluster-level structure and the unidentifiable intra-SCC structure.
We pick the first-stable candidate (lowest spectral radius if no stable one exists); \cref{sec:scalability}, \Cref{app:disjoint-micro} and \Cref{app:grouplingam-exp} use the Hungarian selection of \Cref{alg:condensation}.
By \Cref{thm:condensation-id}, neither choice affects the recovered $\widehat{G}^{\mathrm{sc}}$.
We use threshold $\tau = 0.1$; sensitivity to this choice is reported in \Cref{app:threshold-sensitivity}.

\paragraph{Metrics}

Ground truth is the true SCC partition of $G$ (Tarjan on the data-generating $B$) and the corresponding condensation.
All $F_1$ scores are micro-averaged over directed edges.
We report:
\begin{enumerate}[label=(\roman*)]
  \item \emph{Adjusted Rand Index} (ARI) \citep{hubert1985comparing} of the predicted partition versus the true SCC partition;
  \item \emph{cluster-DAG $F_1$}: $F_1$ of the predicted edges projected onto the true SCC partition against the edges of $G^{\mathrm{sc}}$ (the identifiable target of \Cref{thm:condensation-id}); projected onto the true partition, it is partly oracle and can be substantial while ARI is still low;
  \item \emph{variable-level $F_1$}: $F_1$ over all directed edges (including intra-SCC), quantifying the gap between identifiable and unidentifiable structure.
\end{enumerate}

\subsection{Results}
\label{sec:results}

\Cref{fig:results} shows partition and edge recovery as functions of sample size for both regimes and all three densities, faceted by $\kappa$; three observations:
\begin{itemize}
  \item \emph{Partition recovery.}
        ARI rises with $n$ toward 1 in both regimes and at all densities, confirming consistency of LiNG-D + Tarjan.
  \item \emph{Cluster-DAG $F_1$.}
        Cluster-DAG $F_1$ converges to 1 with $n$ in both regimes, directly illustrating \Cref{thm:condensation-id}: $G^{\mathrm{sc}}$ is identifiable regardless of stability.
  \item \emph{Variable-level $F_1$.}
        In the stable regime (solid), the true $B$ is stable, so the first-stable filter can recover it, and variable-level $F_1 \to 1$.
        In the unstable regime (dashed), the true $B$ has $\rho(B) = 1.5$, so the filter need not select the right equivalence-class member, and variable-level $F_1$ saturates below 1---the unidentifiable intra-SCC structure.
\end{itemize}
The condensation (rows 1--2) is thus robustly recovered even when the variable-level graph (row 3) is not---motivating the cluster-level treatment of \cref{sec:effects}.

\begin{figure}[t]
  \centering
  \includegraphics[width=\textwidth]{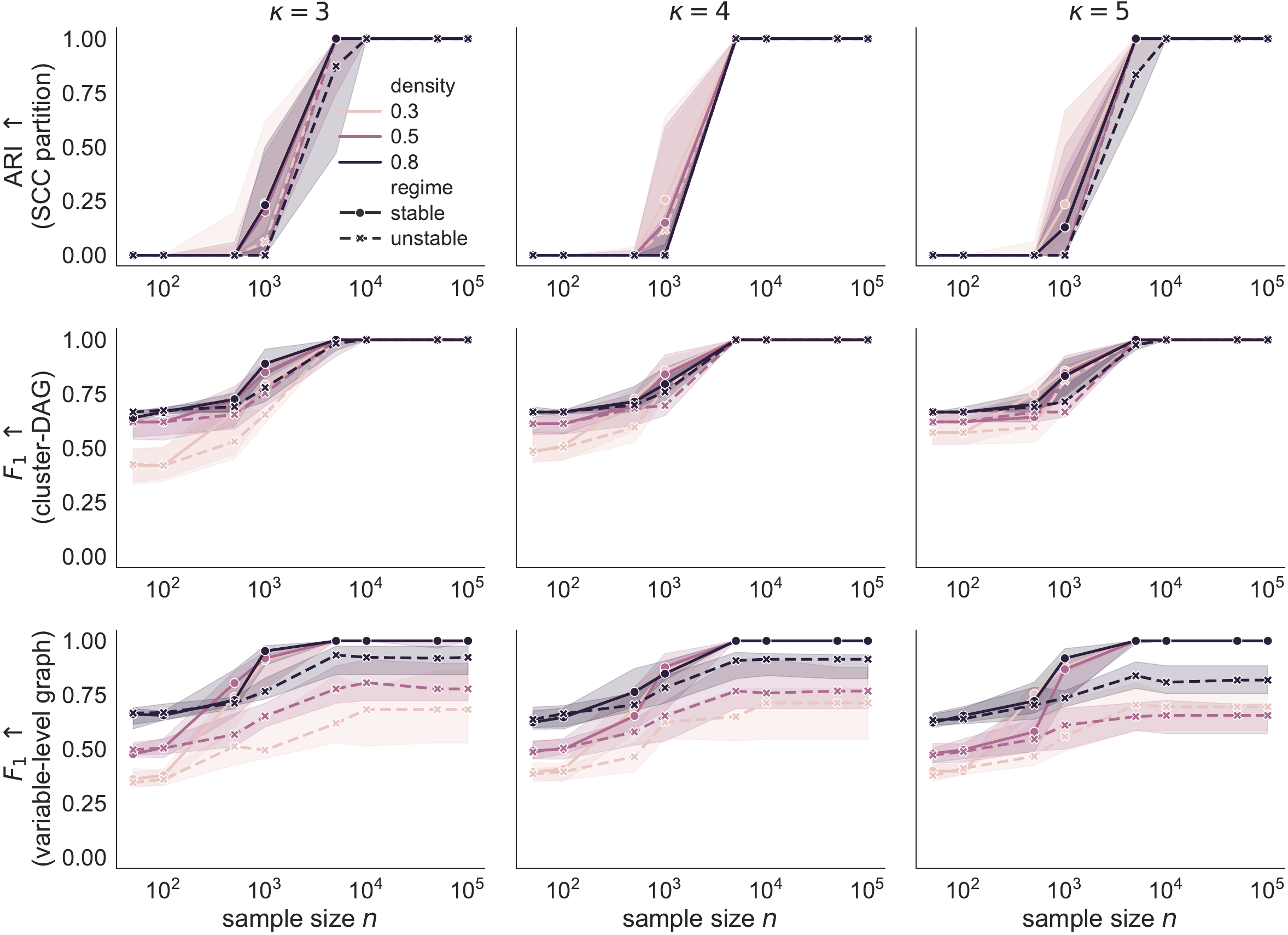}
  \caption{
    Synthetic-data evaluation of the first-stable ICA-LiNG-D recipe ($d = 10$, $10$ seeds per cell).
    Columns: number of SCCs $\kappa$.
    Densities $\lambda \in \{0.3, 0.5, 0.8\}$ in progressively darker shades.
    Solid: stable regime ($\rho(B) \approx 0.9$); dashed: unstable ($\rho(B) = 1.5$).
    Rows: ARI, cluster-DAG $F_1$, variable-level $F_1$.
    Shaded bands: $95\%$ CI across seeds.
  }
  \label{fig:results}
\end{figure}

\subsection{Scalability vs disjoint-cycle baseline}
\label{sec:scalability}

We compare against~\cite{pmlr-v286-drton25a}, which recovers the full variable-level adjacency when all cycles in $G$ are pairwise disjoint, using third-order moments rather than ICA and additionally requiring skewed noise---but the closest baseline available.
We project both methods onto our cluster-level metrics, reading the layered topological ordering of their procedure (layers are SCCs, singletons or cycles) directly as the predicted partition, and sweep $d \in \{20, 50, 100\}$ at density $\lambda = 0.5$ with $\kappa = 10$, $n \in [10^2, 10^5]$, 10 seeds per cell, using skewed exponential noise so that the baseline's third-moment tests carry signal.
\Cref{fig:scalability} shows the result: our method saturates cluster-DAG $F_1$ at $n \geq 10^4$ for every $d$, with ARI saturating by $n \geq 10^5$, while the baseline recovers the partition only at $d = 20$ and plateaus below $0.3$ in cluster-DAG $F_1$ for $d \in \{50, 100\}$ even at $n = 10^5$; fit time grows steeply for the baseline, reaching $10^4$\,s at $n = 10^5$, $d = 100$, ours within $10^1$\,s---consistent with the $\mathcal{O}(knd^2 + d^3)$ bound of \Cref{thm:time-complexity}.
Only at $d = 20$ does the sweep satisfy the baseline's disjoint-cycles assumption: $d = 2\kappa$ holds every SCC at its minimum two variables, so the cycles are pairwise disjoint 2-cycles.
At $d \geq 50$ the larger SCCs carry overlapping intra-SCC cycles, so its plateau stems from the violated assumption, not from sample or computational limits; \Cref{app:disjoint-micro} isolates this with a strictly disjoint micro experiment.

\begin{figure}[t]
  \centering
  \includegraphics[width=\textwidth]{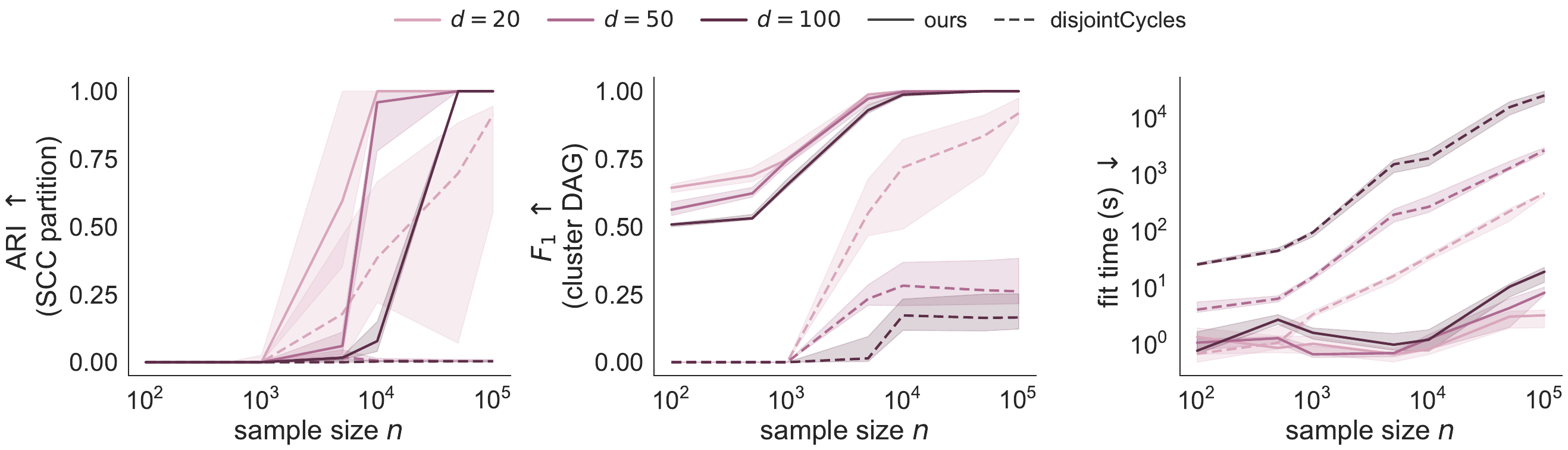}
  \caption{
    Scalability vs the disjoint-cycle baseline of~\cite{pmlr-v286-drton25a} ($d \in \{20, 50, 100\}$, $\lambda = 0.5$, $\kappa = 10$ SCCs, 10 seeds per cell).
    $d$ encoded by color intensity; solid: ours, dashed: disjointCycles.
    Left to right: ARI, cluster-DAG $F_1$, fit time (s).
  }
  \label{fig:scalability}
\end{figure}

\section{Discussion}
\label{sec:discussion}

We have extended the coarsening framework of~\cite{madaleno2026coarsening} to directed graphs with cycles: the DAG-coarsenings of a cyclic graph form a lattice under refinement whose least element---the structural floor---is the SCC partition; its quotient is the condensation $G^{\mathrm{sc}}$.
\Cref{thm:condensation-id} refines the cyclic-SCM identifiability picture of~\cite{lacerda2008discovering}: the micro-graph is identifiable only up to a distributional equivalence class, but $G^{\mathrm{sc}}$ is fully identifiable from the observational distribution.

\paragraph{Limitations}

The observational approach requires linear structural equations, mutually independent non-Gaussian noise with at most one Gaussian component, and no latent confounding, with the practical advantage of requiring no interventional data.
Recovery also depends on a suitable edge threshold $\tau$: too small a threshold merges distinct SCCs, whereas too large a threshold splits them, and the error-free interval in between widens with sample size (\Cref{app:threshold-sensitivity}).
Hard $\mathrm{do}$-interventions on a single variable inside a non-trivial SCC are excluded throughout---they would sever the cyclic mechanism; only soft (shift) interventions and whole-SCC hard interventions remain well-defined (\cref{sec:effects}).
Finally, we do not evaluate on real data: to our knowledge, no public cyclic dataset has an accepted ground-truth graph at this level of detail.
Though protein-signaling networks are thought to contain cycles \citep{sachs2005causal,sachs2009learning}, their consensus networks taken as ground truth for evaluation are DAGs.

\paragraph{Open directions}

(i) Latent confounding could be admitted via overcomplete ICA \citep{salehkaleybar2020learning,dai2026distributional}; the equivalence class then takes a different form, and whether the condensation invariance of \Cref{thm:condensation-id} extends---perhaps via the transitive-closure perspective (\Cref{app:transitive-closure})---is open.
(ii) For the mixed-Gaussian extension, ICA's rotational ambiguity in the Gaussian subspace is global, not per-SCC: an orthogonal rotation can mix two Gaussian sources from different SCCs, introducing inter-SCC entries in the equivalent demixing matrix and falsely merging their SCCs; the invariance argument survives only when at most one noise source is Gaussian \emph{overall}.
Lifting this restriction---thereby addressing the Gaussian/non-Gaussian mixture problem of \citet[\S\ 6]{lacerda2008discovering}---would likely require structural side information beyond ICA.
(iii) A final direction is empirical cluster-level causal-effect estimation via the C-DAG ID algorithm of~\cite{anand2023causal} on $\widehat{G}^{\mathrm{sc}}$, first on semi-synthetic data.

\subsection*{AI use statement}

In this work, we used generative AI tools to assist in deriving the proof of \Cref{prop:finite-sample-rate} (Finite-sample exact recovery) and to support the integration of FastICA and of the disjoint-cycles baseline of \citet{pmlr-v286-drton25a} into our experiments.
We have not used generative AI tools for research ideation or generating synthetic datasets.
Additionally, we used generative AI tools to identify related literature, draft portions of the manuscript, improve writing, edit and format the manuscript, and format and debug code.
We have reviewed all AI-assisted work, including verifying the proof of \Cref{prop:finite-sample-rate}, checking the code for correctness, reviewing the text for accuracy and verifying the references and their relevance to the associated claims.
We take responsibility for the final content of this work, including text, claims, or artifacts produced with the aid of generative AI.

\subsubsection*{Acknowledgments}

The work of FM and FCP was supported by Novo Nordisk Foundation grant number NNF23OC0085356.

The work by AM was supported by Novo Nordisk Foundation grant number NNF20OC0062897 and the Pioneer Centre for Statistical and computational Methods for Advanced Research to Transform Biomedicine (SMARTbiomed), DNRF grant number P4.

The authors thank Carlos Lima Azevedo and Serio Angelo Maria Agriesti for helpful discussions and feedback.
The authors also thank Shohei Shimizu for helpful discussions on the relation between GroupLiNGAM and our approach.


{
  \small
  \bibliographystyle{plainnat} 
  \bibliography{neurips_2026}
}


\appendix
\crefalias{section}{appendix}
\crefalias{subsection}{appendix}

\section{Additional Comments}
\label{app:additional-comments}

\subsection{Relation to GroupLiNGAM}
\label{app:grouplingam}

GroupLiNGAM \citep{kawahara2010grouplingam} learns a partition of the variables into blocks, together with a causal ordering of the blocks, from residual independence tests, assuming between-block independence of the external influences and correlation-faithfulness \citep[Definition~1]{kawahara2010grouplingam}.
Under the assumptions of our linear cyclic SCM (\cref{sec:approach-oica}), it provides an alternative route to the same coarsening.
First, taking the blocks to be the SCCs satisfies the between-block independence assumption.
Second, under faithfulness, GroupLiNGAM's residual-independence criterion identifies exogenous subsets, and its terminal blocks then coincide with the SCCs: if a block is not strongly connected, its condensation contains a proper source SCC that can be split off; conversely, no proper subset of an SCC is exogenous.
The finest GroupLiNGAM block partition is therefore precisely the SCC partition.
Third, the inter-block structure also coincides: once the SCCs are topologically ordered, the adjacency matrix $B$ is block lower triangular, and the regression coefficient from block $i$ to block $j$ is $(I - B_{jj})^{-1} B_{ji}$, where $B_{jj}$ and $B_{ji}$ are the corresponding blocks of $B$; its blockwise zero pattern equals that of $B_{ji}$, so GroupLiNGAM can recover the condensation $G^{\mathrm{sc}}$ as well.

The two approaches differ in assumptions, characterization, and computation.
Exact GroupLiNGAM requires such faithfulness and a worst-case exponential search over subsets, whereas \Cref{thm:condensation-id} identifies the condensation through its invariance across the distributional equivalence class of~\cite{lacerda2008discovering}, yielding a polynomial-time procedure (\Cref{thm:time-complexity}) that invokes no faithfulness, since it relies on thresholding $\widehat{B}$ rather than on conditional-independence tests.
Conversely, GroupLiNGAM is the more general tool: its model permits unobserved confounders within blocks, so its blocks may be coarser than graph-theoretic SCCs \citep{kawahara2010grouplingam}.
\Cref{app:grouplingam-exp} compares the two empirically.

\subsection{On Interventional Extensions}
\label{app:interventional}

Under soft interventions, the cyclic SCM becomes $X^{I} = (I - B)^{-1}(\varepsilon + \delta)$, so intervention effects propagate through the whole graph, including back through cycles.
Nodes within the same SCC are mutual ancestors and therefore share identical intervened-ancestor signatures, so the \texttt{RefineTest} of~\cite{madaleno2026coarsening} cannot split an SCC: \texttt{RePaRe} recovers $G^{\mathrm{sc}}$ with the SCC partition as its floor without modification.

\section{Additional Proofs}
\label{app:proofs}

\subsection{Proof of \Cref{prop:scc-floor}}
\label{app:scc-floor-proof}

\begin{proof}
  If $u$ and $v$ are in the same SCC, there exist directed paths $u \to \cdots \to v$ and $v \to \cdots \to u$ in $G$.
  If $\chi(u) \neq \chi(v)$ in a coarsening, these paths induce directed paths from $\chi(u)$ to $\chi(v)$ and from $\chi(v)$ to $\chi(u)$ in the quotient; their concatenation is a closed directed walk, which contains a directed cycle, violating the DAG requirement.
  Hence, no DAG-coarsening can split an SCC.
  Suppose $G^{\mathrm{sc}}$, the quotient of the SCC coarsening, contained a directed cycle through distinct SCCs.
  The edges on this cycle, together with the directed paths inside each SCC, would make all vertices in their union mutually reachable.
  Their union would therefore be strongly connected, contradicting the maximality of the SCCs.
  Thus, $G^{\mathrm{sc}}$ is acyclic, so the SCC coarsening is a DAG-coarsening.
  Finally, every DAG-coarsening of $G$ groups whole SCCs, so its partition is obtained by coarsening the SCC partition.
  Conversely, every valid coarsening of $G^{\mathrm{sc}}$ yields a DAG-coarsening of $G$ by merging some SCCs.
  DAG-coarsenings of $G$ therefore correspond bijectively to the valid coarsenings of the DAG $G^{\mathrm{sc}}$, with the refinement order preserved in both directions.
  By Proposition~24 of~\cite{tedescoScoreBasedLearningCluster2026a}, they form a lattice under refinement.
  The SCC coarsening refines every DAG-coarsening and is therefore the least element of this lattice.
\end{proof}

\subsection{Proof of \Cref{thm:condensation-id}}
\label{app:condensation-id-proof}

The proof rests on a purely graph-theoretic lemma that requires no structural assumption on $G$---in particular, no disjointness of cycles.
It plays the role that the cycle-reversal characterization of \citet[Theorem~4]{lacerda2008discovering} plays under their disjoint-cycle assumption; cf.\ the equivalent bipartite-matching view of these permutations in \citet{sharifian2025}.

\begin{lemma}[Admissible permutations preserve the condensation]
  \label{lem:perm-condensation}
  Let $G = (V, E)$ be a directed graph and let $\sigma\colon V \to V$ be a bijection such that,
  for every $i \in V$, either $\sigma(i) = i$ or $\sigma(i) \to i$ is an edge of $G$.
  Let $G_\sigma$ be the directed graph on $V$ with edge set
  \[
    E_\sigma \; = \; \bigl\{ \, j \to \sigma(i) \;:\; i \in V,\ j \in (\{i\} \cup \operatorname{pa}(i)) \setminus \{\sigma(i)\} \, \bigr\},
  \]
  where $\operatorname{pa}(i)$ denotes the parents of $i$ in $G$.
  Then $G_\sigma$ and $G$ have the same SCC partition and the same condensation.
\end{lemma}

\begin{proof}
  Write $[\, \cdot \,]$ for the SCC of a vertex, and order components by reachability, so that $[u] \preceq [v]$ if and only if $u$ reaches $v$;
  distinct components have distinct ranks in any linear extension of this order.

  \emph{$\sigma$ preserves every SCC.}
  By admissibility, $\sigma(i)$ reaches $i$ for every $i$, that is, $[\sigma(i)] \preceq [i]$.
  Let $r(\cdot)$ be the rank in a linear extension.
  Since $\sigma$ is a bijection, $\sum_i r(\sigma(i)) = \sum_i r(i)$, which together with the termwise inequalities $r(\sigma(i)) \le r(i)$ forces $r(\sigma(i)) = r(i)$ for all $i$; hence $[\sigma(i)] = [i]$, and $\sigma$ maps each SCC onto itself.

  \emph{The SCC partition is preserved.}
  First, every $G_\sigma$-SCC is contained in an SCC of $G$: an edge $j \to \sigma(i)$ of $G_\sigma$ with $[j] \neq [i]$ projects to the component pair $([j], [i])$ of the $G$-edge $j \to i$, so a $G_\sigma$-closed walk visiting several components of $G$ would induce a closed walk in the condensation of $G$, which is acyclic.
  Second, no SCC of $G$ splits.
  If $|C| = 1$ there is nothing to show; otherwise, let $S \subset C$, $S \neq C$, be a source block of the acyclic quotient of $G_\sigma$ inside $C$, so no $G_\sigma$-edge runs from $C \setminus S$ into $S$.
  For $v = \sigma(i) \in S$, the in-neighbors of $v$ inside $C$ are $(\{i\} \cup \operatorname{pa}_C(i)) \setminus \{v\}$, where $\operatorname{pa}_C(i) = \operatorname{pa}(i) \cap C$, and all of them lie in $S$; if $i \neq v$ this set contains $i$, since $i \to \sigma(i) \in E_\sigma$, and otherwise $i = v \in S$.
  Hence $\sigma^{-1}(S) \subseteq S$, and since $\sigma$ is a bijection, $\sigma^{-1}(S) = S$ and $\sigma(S) = S$.
  Now for $i \in S$, its image $v = \sigma(i)$ lies in $S$, so $(\{i\} \cup \operatorname{pa}_C(i)) \setminus \{v\} \subseteq S$ and therefore $\operatorname{pa}_C(i) \subseteq S$: $S$ is parent-closed in $C$.
  Finally, since $C$ is strongly connected, there is a $G$-path from $C \setminus S$ into $S$, and all of its vertices lie in $C$ by maximality of SCCs; its last edge $x \to y$ has $y \in S$ and $x \in C \setminus S$, so $y$ has a parent outside $S$---contradicting parent-closedness.
  Hence each $C$ is a single $G_\sigma$-SCC, and the partitions coincide.

  \emph{The condensations coincide.}
  The inter-component edges of $G_\sigma$ are exactly $\{\, j \to \sigma(i) : j \to i \in E,\ [j] \neq [i] \,\}$; here $j \neq i$ is automatic.
  Since $\sigma$ maps each component onto itself, $i \mapsto \sigma(i)$ bijects these edges with the inter-component edges $j \to i$ of $G$, preserving the component pair $([j], [i])$.
  Equal partitions and equal inter-component edge sets give the same condensation.
\end{proof}

\begin{proof}[Proof of \Cref{thm:condensation-id}]
  Let $G_1, G_2$ underlie models of the form~\eqref{eq:sem} that induce the same observational distribution, with demixing matrices $W_1 = I - B_1$ and $W_2 = I - B_2$.
  By the ICA identifiability fact stated in \Cref{sec:approach-oica} (mutually independent noises, at most one Gaussian), $W_2 = D P W_1$ for a permutation matrix $P$ and a nonsingular diagonal matrix $D$; write $\sigma$ for the vertex map induced by $P$, so that row $i$ of $W_1$ lands in row $\sigma(i)$ of $P W_1$.
  Comparing diagonal entries gives $1 = D_{\sigma(i)\sigma(i)} W_{1,i\sigma(i)}$ for every $i$: the diagonal of $P W_1$ is entrywise nonzero---that is, $\sigma(i) = i$ or $\sigma(i) \to i \in E_1$---and $D = \operatorname{diag}(P W_1)^{-1}$.
  Since row $i$ of $W_1$ has support $\{i\} \cup \operatorname{pa}_1(i)$, the off-diagonal support of $B_2 = I - \operatorname{diag}(P W_1)^{-1} P W_1$ has, in row $\sigma(i)$, exactly $(\{i\} \cup \operatorname{pa}_1(i)) \setminus \{\sigma(i)\}$: $G_2 = (G_1)_\sigma$ in the sense of \Cref{lem:perm-condensation}, and the lemma gives $G_1^{\mathrm{sc}} = G_2^{\mathrm{sc}}$.
\end{proof}

\subsection{Proof of \Cref{lem:cond-faith}}
\label{app:cond-faith-proof}

\begin{proof}
  Linear cyclic SCMs are $\sigma$-Markov to $G$ \citep{spirtes1995directed,bongers2021foundations}.
  Aggregating variables to their SCCs projects $\sigma$-separation statements about $G$ onto separation statements about $G^{\mathrm{sc}}$; since $G^{\mathrm{sc}}$ is acyclic, $\sigma$-separation collapses to ordinary d-separation, so the cluster-level distribution is Markov to $G^{\mathrm{sc}}$.
\end{proof}

\subsection{Proof of \Cref{thm:time-complexity}}
\label{app:time-complexity}

\begin{proof}

  The overall time complexity of \Cref{alg:condensation} is governed by the sequential execution of statistical estimation and graph-theoretic aggregation.
  We break down the worst-case asymptotic time complexity by stage:

  \textbf{Independent Component Analysis (FastICA):} FastICA computes the unmixing matrix $\widehat{W}$ iteratively.
  The computational cost per iteration is dominated by matrix multiplications involving the $n \times d$ data matrix, taking $\mathcal{O}(k \cdot n \cdot d^2)$ time, where $k$ represents the number of iterations until convergence.

  \textbf{Admissible-permutation search (combinatorial bypass).}
  A candidate adjacency requires a row permutation $P$ making the diagonal of $P\widehat{W}$ nonzero.
  Exhaustive enumeration over the distributional equivalence class incurs an N-rooks worst-case cost of $\mathcal{O}(d!)$ (in practice typically much smaller, depending on sparsity of $\widehat{W}$).
  By \Cref{thm:condensation-id}, any admissible $P$ yields the same $G^{\mathrm{sc}}$, so finding one $P$ suffices: a single Hungarian solve on the forbidden-cell cost returns a permutation with all matched entries $> \eta$ in $\mathcal{O}(d^{3})$ whenever one exists.

  \textbf{Candidate Generation and Thresholding:} Once an admissible permutation $P$ is found, computing $\widehat{B}_P = I - \text{diag}(P\widehat{W})^{-1}P\widehat{W}$ and applying element-wise thresholding takes $\mathcal{O}(d^2)$ time.
  If optional stability filtering is applied, computing the spectral radius $\rho(\widehat{B}_P)$ via eigenvalue decomposition bounds this step at $\mathcal{O}(d^3)$.

  \textbf{SCC Partitioning:} Extracting the Strongly Connected Components via Tarjan's algorithm runs in linear time with respect to the candidate graph's vertices and edges, taking $\mathcal{O}(d + |E|) \subseteq \mathcal{O}(d^2)$ time.

  \textbf{Condensation Edge Formation:} Scanning the edge list to establish inter-cluster adjacencies requires a single pass over the nonzero elements of the adjacency matrix, taking $\mathcal{O}(|E|) \subseteq \mathcal{O}(d^2)$ time.

  Summing these stages yields an overall worst-case time complexity strictly bounded by $\mathcal{O}(k \cdot n \cdot d^2 + d^3)$.
\end{proof}

\subsection{Transitive closure and SCCs}
\label{app:transitive-closure}

For a directed graph $G = (V, E)$, denote the existence of a directed path from vertex $u$ to vertex $v$ by the reachability relation $u \leadsto_G v$.
By convention, a vertex can always reach itself (paths of length 0), so $u \leadsto_G u$ for all $u \in V$.

\begin{definition}[Strongly connected component]
  \label{def:scc}
  Define a binary relation $\sim_G$ on $V$ by
  \[
    u \sim_G v \iff (u \leadsto_G v) \land (v \leadsto_G u).
  \]
  The relation $\sim_G$ is an equivalence relation (reflexive, symmetric, transitive).
  A \emph{strongly connected component} (SCC) of $G$ is an equivalence class of $V$ under $\sim_G$.
\end{definition}

The following result provides an alternative route to condensation identifiability via the mixing matrix $A = (I - B)^{-1}$ and its transitive-closure interpretation.
While the main proof of \Cref{thm:condensation-id} proceeds by admissible permutations (\Cref{lem:perm-condensation}), the transitive-closure perspective is independently useful---for instance, it connects to the path-rank lemma of~\cite{dai2026distributional} and to settings where one works with $A$ rather than $W$.

\begin{definition}[Transitive closure]
  \label{def:transitive-closure}
  The \emph{transitive closure} of a directed graph $G = (V, E)$ is the graph $G^{\mathrm{tc}} = (V, E^{\mathrm{tc}})$ with edge set
  \[
    E^{\mathrm{tc}} = \{(u,v) \in V \times V \mid u \leadsto_G v\}.
  \]
\end{definition}

\begin{theorem}[SCCs are preserved under transitive closure]
  \label{thm:scc-tc}
  Let $G = (V, E)$ be a directed graph and let $G^{\mathrm{tc}} = (V, E^{\mathrm{tc}})$ be its transitive closure.
  Then $G$ and $G^{\mathrm{tc}}$ have the same SCCs.
\end{theorem}

\begin{proof}
  It suffices to show $u \sim_G v \iff u \sim_{G^{\mathrm{tc}}} v$ for all $u, v \in V$.

  \noindent($\Rightarrow$) Suppose $u \sim_G v$, so $u \leadsto_G v$ and $v \leadsto_G u$.
  By \Cref{def:transitive-closure}, $(u, v) \in E^{\mathrm{tc}}$ and $(v, u) \in E^{\mathrm{tc}}$, hence $u \leadsto_{G^{\mathrm{tc}}} v$ and $v \leadsto_{G^{\mathrm{tc}}} u$, i.e.\ $u \sim_{G^{\mathrm{tc}}} v$.

  \noindent($\Leftarrow$) Suppose $u \sim_{G^{\mathrm{tc}}} v$, so there is a path $u = x_0, x_1, \dots, x_k = v$ in $G^{\mathrm{tc}}$.
  Each step $(x_i, x_{i+1}) \in E^{\mathrm{tc}}$ implies $x_i \leadsto_G x_{i+1}$ by \Cref{def:transitive-closure}.
  Transitivity of $\leadsto_G$ then gives $u \leadsto_G v$.
  Symmetric reasoning yields $v \leadsto_G u$, so $u \sim_G v$.

  The two implications give equality of equivalence classes, hence equality of SCCs.
\end{proof}

\begin{remark}[the transitive-closure route]
  \label{rem:transitive-closure}
  \Cref{thm:scc-tc} means that if one recovers $\operatorname{supp}(A) = \operatorname{adjacency}(G^{\mathrm{tc}})$---the transitive closure of $G$---from the mixing matrix $A = (I - B)^{-1}$, the SCC partition extracted from it is exactly the SCC partition of $G$.
  This provides an alternative proof route for the SCC-partition part of \Cref{thm:condensation-id}, going through $A$ rather than $W$.
  The main text uses the permutation argument (\Cref{lem:perm-condensation}) because it also immediately handles inter-SCC edges without requiring a separate step.
\end{remark}

\subsection{Sample complexity for exact condensation recovery}
\label{app:sample-complexity}

\Cref{cor:condensation-id} guarantees that \Cref{alg:condensation} returns the correct condensation \emph{in the large-sample limit}.
This subsection turns that asymptotic statement into a finite-sample bound on the sample complexity, that is, how many observations $n$ are sufficient to exactly recover $G^{\mathrm{sc}}$ with probability at least $1 - \delta$.

\paragraph{Setup}

Fix a linear non-Gaussian SCM $X = BX + \varepsilon$ with mutually independent noise components, at most one Gaussian, and $\det(I - B) \neq 0$.
Let $\widehat{W}_n$ be the FastICA estimate from $n$ i.i.d.\ samples and let $\widehat{B}_n$ be the candidate produced by \Cref{alg:condensation} at thresholds $\tau_n$ and $\eta$.

\paragraph{Candidate consistency}

ICA identifies $W = I - B$ only up to row permutation and scaling, and the assignment step may select any admissible representative, so $\widehat{B}_n$ need not concentrate around $B$ itself: its population limit is the candidate \emph{matched} to the ICA output, which may be a different admissible representative of the distributional equivalence class.
A bound stated in terms of $\operatorname{supp}(\widehat{B}_n) \neq \operatorname{supp}(B)$ therefore need not vanish with $n$---indeed, in the unstable regime of \Cref{fig:results}, variable-level $F_1$ saturates below one while the condensation is still recovered.
We instead assume a fourth-moment form of $\sqrt n$-consistency for the two-step estimator (ICA followed by the assignment step), centered at its matched population target.

\begin{assumption}[Candidate consistency]
  \label{assum:candidate-consistency}
  There exist a valid population limit $W^\star$ of FastICA---that is, $W^\star = D\, P_0 (I - B)$ for a nonsingular diagonal $D$ and a permutation $P_0$, not a spurious population fixed point---an admissible permutation $P^\star$ of $W^\star$, and constants $M, C > 0$ and $n_0$ such that, writing
  \[
    B^\star := I - \operatorname{diag}(P^\star W^\star)^{-1} P^\star W^\star, \qquad \tilde B_n := I - \operatorname{diag}(P^\star \widehat{W}_n)^{-1} P^\star \widehat{W}_n,
  \]
  for the corresponding candidate~\eqref{eq:B-from-W} at $P^\star$ and for the raw candidate at $P^\star$ before thresholding, and
  \[
    \beta_{\min} := \min_{(i,j)\,:\,B^\star_{ij} \neq 0} |B^\star_{ij}|
  \]
  for the smallest nonzero entry of $B^\star$ (the edgeless case $B^\star = 0$ is vacuous and excluded), the following conditions hold for all $n \geq n_0$:
  \begin{enumerate}[label=(\roman*)]
    \item the entry-wise tail bound
          \[
            \Pr\bigl[ \, |\tilde B_{n,ij} - B^\star_{ij}| \geq s \, \bigr] \; \leq \; \frac{M}{n^{2}\, s^{4}}
          \]
          for all entries $(i,j)$ and all $0 < s \leq 2\beta_{\min}$;
    \item the assignment step of \Cref{alg:condensation} returns $P^\star$ with $|(P^\star \widehat{W}_n)_{ii}| > \eta$ for every $i$, with probability at least $1 - C/n^{2}$.
  \end{enumerate}
  We provide further justification and discussion of this assumption in \Cref{rem:candidate-consistency}.
\end{assumption}

FastICA's row order and signs vary across runs, so $\widehat{W}_n$ is aligned to $W^\star$ up to row permutation and sign; \Cref{alg:condensation}'s output is invariant to both (the assignment is permutation-equivariant, and the log-weights and the candidate map are sign-invariant), and the unit-variance convention pins the row scalings, so the alignment is without loss of generality.
We write $\mathcal{E}_{\mathrm{asg},n}$ for the assignment event of \Cref{assum:candidate-consistency} and count runs outside it as recovery failures: there the candidate is either undefined (the algorithm halts) or uncontrolled (the assignment returns a different feasible permutation).
Runs on which $\tilde B_n$ is undefined---$(P^\star \widehat{W}_n)_{ii} = 0$ for some $i$---fail condition \emph{(ii)}'s every-diagonal requirement and are therefore counted in the assignment failure $\mathcal{E}_{\mathrm{asg},n}^c$.
Any such $B^\star$ is the candidate~\eqref{eq:B-from-W} of the admissible permutation $Q = P^\star P_0$ applied to $I - B$ (the candidate map is invariant to row scaling), so \Cref{lem:perm-condensation} gives that the graph induced by $\operatorname{supp}(B^\star)$ shares the SCC partition and the condensation of $G$; any candidate satisfying \Cref{assum:candidate-consistency} yields an equally valid bound, though $\beta_{\min}$ may differ.

\begin{proposition}[Finite-sample exact recovery]
  \label{prop:finite-sample-rate}
  Assume \Cref{assum:candidate-consistency}.
  Let $K_1 := d^{2} M + C\,\beta_{\min}^{4}$ and $K_2 := d^{2} M$.
  For every $n \geq n_0$ and every threshold $\tau_n \in (0, \frac{\beta_{\min}}{2}]$,
  the recovered condensation satisfies
  \begin{equation}
    \label{eq:finite-sample-bound}
    \begin{split}
      \Pr\bigl[ \, \widehat{G}^{\mathrm{sc}}_n \neq G^{\mathrm{sc}} \, \bigr] & \; \leq \; \Pr\bigl[ \, \operatorname{supp}(\widehat{B}_n) \neq \operatorname{supp}(B^\star) \,\text{ or the run falls outside $\mathcal{E}_{\mathrm{asg},n}$} \, \bigr] \\
                                                                              & \; \leq \; \frac{K_1}{n^{2}\, (\beta_{\min} - \tau_n)^{4}} \; + \; \frac{K_2}{n^{2}\, \tau_n^{4}},
    \end{split}
  \end{equation}
  and the right-hand side vanishes as $n \to \infty$ if and only if $n^{2}\, \tau_n^{4} \to \infty$.
\end{proposition}

The two terms of \eqref{eq:finite-sample-bound} mirror the failure modes of \Cref{prop:conservative}: a split requires dropping an intra-SCC entry that breaks reachability (first term) while a merge requires keeping a cross-SCC entry that closes a path (second term).
\begin{proof}
  The proof has three steps.

  \emph{Step 1 (sup-norm concentration).}
  By \Cref{assum:candidate-consistency} and a union bound over the $d^{2}$ entries,
  \[
    \Pr\bigl[ \, \|\tilde B_n - B^\star\|_{\infty} \geq s \, \bigr] \; \leq \; \sum_{i,j} \Pr\bigl[ \, |\tilde B_{n,ij} - B^\star_{ij}| \geq s \, \bigr] \; \leq \; \frac{d^{2}\, M}{n^{2}\, s^{4}}
  \]
  for all $0 < s \leq 2\beta_{\min}$ and $n \geq n_0$, where $\| \cdot \|_\infty$ denotes the entrywise maximum norm.

  \emph{Step 2 (support recovery via two opposing tails).}
  On $\mathcal{E}_{\mathrm{asg},n}$, the algorithm returns $P^\star$, and its candidate is $\widehat{B}_n$, the entrywise thresholding of $\tilde B_n$ at $\tau_n$.
  \begin{itemize}
    \item \emph{False negative.}
          If a nonzero entry ($|B^\star_{ij}| \geq \beta_{\min}$) is dropped, i.e.\ $|\tilde B_{n,ij}| < \tau_n$, the triangle inequality gives $\|\tilde B_n - B^\star\|_{\infty} > \beta_{\min} - \tau_n$; by Step 1, the probability is at most $\frac{d^{2} M}{n^{2}\, (\beta_{\min} - \tau_n)^{4}}$.
    \item \emph{False positive.}
          If a zero entry ($B^\star_{ij} = 0$) is kept, i.e.\ $|\tilde B_{n,ij}| \geq \tau_n$, then $\|\tilde B_n - B^\star\|_{\infty} \geq \tau_n$, so the probability is at most $\frac{d^{2} M}{n^{2}\, \tau_n^{4}}$.
  \end{itemize}
  If neither event occurs, every nonzero entry of $B^\star$ survives thresholding and every zero entry is removed, so $\operatorname{supp}(\widehat{B}_n) = \operatorname{supp}(B^\star)$.
  A union bound over the two events, together with the assignment failure $\mathcal{E}_{\mathrm{asg},n}^{c}$ of \Cref{assum:candidate-consistency} (absorbed into the first term: $C/n^{2} = C\beta_{\min}^{4}/(n^{2}\beta_{\min}^{4}) \leq C\beta_{\min}^{4}/(n^{2}(\beta_{\min} - \tau_n)^{4})$), yields the right-hand side of \eqref{eq:finite-sample-bound} with the constants stated above.

  \emph{Step 3 (from $\operatorname{supp}(B^\star)$ to the condensation).}
  As noted above, the directed graph induced by $\operatorname{supp}(B^\star)$ has the SCC partition of $G$ and condensation $G^{\mathrm{sc}}$.
  Hence, on $\{\operatorname{supp}(\widehat{B}_n) = \operatorname{supp}(B^\star)\}$, Tarjan's algorithm returns the exact partition and condensation.
  This step is deterministic: on $\mathcal{E}_{\mathrm{asg},n}$, $\widehat{G}^{\mathrm{sc}}_n \neq G^{\mathrm{sc}}$ implies $\operatorname{supp}(\widehat{B}_n) \neq \operatorname{supp}(B^\star)$, that is,
  \[
    \{\, \widehat{G}^{\mathrm{sc}}_n \neq G^{\mathrm{sc}} \,\} \; \subseteq \; \mathcal{E}_{\mathrm{asg},n}^{c} \cup \{\, \operatorname{supp}(\widehat{B}_n) \neq \operatorname{supp}(B^\star) \,\},
  \]
  which combined with Step 2 yields \eqref{eq:finite-sample-bound}.
  \qedhere
\end{proof}

\paragraph{Sample complexity}

The two terms in \eqref{eq:finite-sample-bound} are balanced when $\tau_n = \frac{\beta_{\min}}{2}$, in which case both reduce to a constant times $\frac{1}{(n^{2}\, \beta_{\min}^{4})}$.
Inverting the bound yields:

\begin{corollary}[Sample complexity]
  \label{cor:sample-complexity}
  Under the assumptions of \Cref{prop:finite-sample-rate}, fix the threshold to $\tau = \frac{\beta_{\min}}{2}$.
  Then for any target error level $\delta \in (0, 1)$, taking
  \[
    n \; \geq \; \max\biggl\{ n_0,\; \frac{4}{\beta_{\min}^{2}}\, \sqrt{\frac{K_1 + K_2}{\delta}} \biggr\}
  \]
  samples is sufficient for
  \[
    \Pr\bigl[ \, \widehat{G}^{\mathrm{sc}}_n = G^{\mathrm{sc}} \, \bigr] \; \geq \; 1 - \delta.
  \]
\end{corollary}

\begin{proof}
  By \Cref{prop:finite-sample-rate}, \eqref{eq:finite-sample-bound} holds at $\tau = \frac{\beta_{\min}}{2}$ for every $n \geq n_0$.
  Substituting $\tau_n = \frac{\beta_{\min}}{2}$ bounds the sum of the two terms by $\frac{16(K_1 + K_2)}{n^{2}\, \beta_{\min}^{4}}$.
  Requiring this to be at most $\delta$ and solving for $n$ gives the stated bound.
\end{proof}

We make three comments on \Cref{cor:sample-complexity}.

First, the corollary's sufficient sample size is $n = \frac{4\sqrt{K_1 + K_2}}{\beta_{\min}^{2}\, \sqrt\delta}$; unpacking $K_1 + K_2 = 2d^{2} M + C\,\beta_{\min}^{4}$ gives $n = O\bigl( \frac{d\sqrt{M}}{\beta_{\min}^{2}\, \sqrt\delta} + \frac{\sqrt{C}}{\sqrt\delta} \bigr)$, which is quadratic in the inverse minimum-edge magnitude with $\frac{1}{\sqrt\delta}$ dependence on the failure probability, up to the constants of \Cref{assum:candidate-consistency}, whose dimension dependence (polynomial in $d$, $1/\gamma$, $1/\eta$, and $1/r$) we do not track further.
If exponential-type tail bounds for the estimator are available, the polynomial $\delta$-dependence would tighten to a polylogarithmic one, but a fourth-moment bound on the estimator alone yields the polynomial rate above.

Second, the choice $\tau = \frac{\beta_{\min}}{2}$ requires a priori knowledge of $\beta_{\min}$, which we do not have in practice; any threshold sequence with $\tau_n \to 0$ and $\sqrt n\, \tau_n \to \infty$ yields failure probability $O\bigl( n^{-2} + (\sqrt{n}\,\tau_n)^{-4} \bigr)$---for instance $\tau_n \asymp 1/\log n$ gives $O\bigl( (\log n)^{4}/n^{2} \bigr)$---though the $n$ from which the guarantee holds still depends on $\beta_{\min}$ through $\tau_n \leq \frac{\beta_{\min}}{2}$.

Third, the bound \eqref{eq:finite-sample-bound} reproduces the threshold-sensitivity picture of \Cref{app:threshold-sensitivity} qualitatively: shrinking $\tau$ inflates the second term (false positives, Mode (b), merging) and growing $\tau$ inflates the first (false negatives, Mode (a), splitting); at $\tau = 0.1$ with $\beta_{\min} \approx 0.51$ (the regime of \Cref{fig:sample-complexity}), the false-positive term's denominator is $\bigl( \frac{\beta_{\min} - \tau}{\tau} \bigr)^{4} \approx 280$ times smaller, consistent with every observed error at the default threshold being a merge.
The bound controls the support event, which is larger than either mode: a non-closing cross-SCC keep adds a spurious condensation edge without merging, an intra-SCC keep is harmless, and a dropped inter-SCC entry corrupts the edge set of $\widehat{G}^{\mathrm{sc}}_n$ only if no other entry for that cluster pair survives.

\begin{remark}[When does \Cref{assum:candidate-consistency} hold?]
  \label{rem:candidate-consistency}
  \Cref{assum:candidate-consistency} holds under cited FastICA asymptotics strengthened by an assumed uniform fourth-moment bound (condition \emph{(i)}), an assignment-gap condition (condition \emph{(ii)}), and the proviso that FastICA attains a valid population optimum rather than a spurious local one.
  The next three parapraphs treat, in turn, the estimator side of condition \emph{(i)}, the assignment side of condition \emph{(ii)}, and a checkable special case with $B^\star = B$.

  \emph{Estimation:} under standard regularity and moment conditions and at most one Gaussian component, FastICA-type estimators are $\sqrt n$-consistent for the unmixing matrix up to row permutation and sign: the population output is some $W^\star$ and $\sqrt n\,(\widehat{W}_n - W^\star)$ is $O_p(1)$ \citep{chen2005consistent,miettinen2015fourth}; condition \emph{(i)} invokes the stronger fourth-moment form, a uniform bound $m_4 := \sup_{n \geq n_0} \max_{i,j}\, n^{2}\,\mathbb{E}\bigl[ (\widehat{W}_{n,ij} - W^\star_{ij})^{4} \bigr] < \infty$, which we assume and which requires at least finite eighth moments of the noise.

  \emph{Assignment:} write $\operatorname{wt}(P, W) := \sum_i \log |(P W)_{ii}|$.
  The log-assignment of \Cref{alg:condensation}---maximize $\operatorname{wt}(P, W^\star)$ over permutations feasible at $W^\star$ (all matched entries of $P W^\star$ exceed $\eta$)---selects $P^\star$ uniquely: assume
  \[
    \operatorname{wt}(P^\star, W^\star) - \operatorname{wt}(P, W^\star) \; \geq \; \gamma \; > \; 0
  \]
  for every other permutation with all matched entries nonzero (a zero matched entry gives $\operatorname{wt}(P, W^\star) = -\infty$).
  Write $\delta_0 := \min_i |(P^\star W^\star)_{ii}| > 0$, fix a threshold $\eta \in (0, \delta_0)$, and pick $r \in (0, \delta_0 - \eta)$; every matrix in the closed ball of radius $r$ around $P^\star W^\star$ has diagonal entries of magnitude at least $\delta_0 - r$, so the map $W' \mapsto I - \operatorname{diag}(W')^{-1} W'$ is continuously differentiable on a neighborhood of the ball and hence, by compactness, $L$-Lipschitz on it.
  For $P$ feasible at $\widehat{W}_n$ and $\|\widehat{W}_n - W^\star\|_\infty \leq \eta/2$, every matched entry of $P W^\star$ exceeds $\eta/2$---permutations preserve the sup norm, so no population-inadmissible permutation is feasible on the ball---and $|\log|\widehat{W}_{e}| - \log|W^\star_{e}|| \leq 2\,|\widehat{W}_{e} - W^\star_{e}|/\eta$; hence $|\operatorname{wt}(P, \widehat{W}_n) - \operatorname{wt}(P, W^\star)| \leq 2d\,\|\widehat{W}_n - W^\star\|_\infty/\eta$.
  On
  \[
    \mathcal{B}_n := \{\|\widehat{W}_n - W^\star\|_\infty \leq \min(r, \eta/2, \eta\gamma/(8d))\},
  \]
  $P^\star$ is itself feasible, its matched entries having magnitude at least $\delta_0 - r > \eta$, and every feasible $P \neq P^\star$ satisfies
  \[
    \operatorname{wt}(P, \widehat{W}_n) - \operatorname{wt}(P^\star, \widehat{W}_n) \; \leq \; - \gamma + 4d\,\|\widehat{W}_n - W^\star\|_\infty/\eta \; \leq \; - \gamma/2 \; < \; 0,
  \]
  so $P^\star$ is the unique feasible maximizer and the assignment step returns it.
  Since $P^\star \widehat{W}_n$ lies in the ball of radius $r$ around $P^\star W^\star$, moreover $\|\tilde B_n - B^\star\|_\infty \leq L\,\|\widehat{W}_n - W^\star\|_\infty$.
  The complement of $\mathcal{B}_n$ has probability at most $C/n^{2}$ by Markov's inequality on the fourth moments, supplying the constant $C$ of condition \emph{(ii)}.
  The Lipschitz control is ball-local, so the tail of condition \emph{(i)} splits across $\mathcal{B}_n$: inside it, $\Pr\bigl[ \,\{|\tilde B_{n,ij} - B^\star_{ij}| \geq s\} \cap \mathcal{B}_n\, \bigr] \leq d^{2}\, m_4 L^{4}/(n^{2} s^{4})$ by Markov's inequality on the fourth moment of $\|\widehat{W}_n - W^\star\|_\infty$; outside it, the trivial bound is $C/n^{2} \leq C(2\beta_{\min})^{4}/(n^{2} s^{4})$ on the range $0 < s \leq 2\beta_{\min}$ that the proof invokes.
  Both pieces are absorbed into $M$.
  $M$ and $C$ are polynomial in $d$, $1/\gamma$, $1/\eta$, and $1/r$ (for fixed $m_4$ and $L$); we do not track them further.
  The uniqueness of the maximizer in \emph{Assignment} is a proof convenience; by \Cref{lem:perm-condensation}, a tied admissible maximizer yields the same condensation, with possibly different $\beta_{\min}$.

  \emph{A checkable special case:} when $\max_{ij}|B_{ij}| < 1$, the true alignment is the unique maximizer ($P^\star = P_0^{-1}$, hence $B^\star = B$), with $\gamma = \min_{\mathcal{C}} \log\bigl( 1/|\prod_{e \in \mathcal{C}} b_{e}| \bigr) \geq 2\log(1/\max_{ij}|B_{ij}|)$ over directed cycles $\mathcal{C}$, where $b_{e}$ is the weight of edge $e$; the log objective is natural here, because an admissible $\sigma$'s log-weight is, up to a $\sigma$-independent constant, $\sum_{\mathcal{C}} \log|\prod_{e \in \mathcal{C}} b_{e}|$ over the vertex-disjoint cycles $\mathcal{C}$ of $\sigma$, each a directed cycle of $G$ traversed backwards.\footnote{The true alignment is thus the unique population maximizer exactly when every directed cycle has $|\prod b| < 1$; with vertex-disjoint cycles the maximizer reverses precisely the cycles with $|\prod b| > 1$, the unique stable representative of \citet{lacerda2008discovering}.}
\end{remark}

\paragraph{Empirical corroboration}

\Cref{fig:sample-complexity} reports a Monte-Carlo experiment that corroborates the rate of \Cref{prop:finite-sample-rate} on the same data-generating regime as the main-grid experiments of \Cref{sec:experiments} ($d = 10$, $\kappa = 4$ non-trivial SCCs, $\lambda = 0.5$, weights uniform in $[0.5, 0.95]$, Laplace noise, stable regime; $\beta_{\min} \approx 0.51$).
We fix one ground-truth SCM, run the pipeline of \Cref{sec:experiments} with the first-stable filter at $\tau = \frac{\beta_{\min}}{2}$ as prescribed by \Cref{cor:sample-complexity}, and over 300 independent seeds per cell sweep $n \in [10^{2}, 10^{4}]$ on a log-spaced grid.
For this SCM $\max_{ij}|B_{ij}| < 1$, so every directed cycle has $|\prod_{e} b_{e}| < 1$ and $B^{\star} = B$ (\Cref{rem:candidate-consistency}); the first-stable filter and the Hungarian selection of \Cref{alg:condensation} thus return the same candidate in the large-sample limit, and the plotted quantities target \eqref{eq:finite-sample-bound} directly.
The left panel plots $\mathbb{E}[d_H(\operatorname{supp}(\widehat{B}),\, \operatorname{supp}(B))]$, which by linearity is exactly the sum of entry-wise support-error probabilities that the union bound in Step 2 controls, and an upper bound on $1 - \Pr[\text{exact recovery}]$.
The right panel plots the binary exact-recovery probability $\Pr[\operatorname{supp}(\widehat{B}) = \operatorname{supp}(B)]$ with $95\%$ CI bands.
An Ordinary Least Squares (OLS) log-log fit on the cells with non-trivial failure rate gives empirical slopes of $-4.0$ for mean Hamming distance and $-2.9$ for exact-recovery error in the transition window, both steeper than the $n^{-2}$ rate of the upper bound \eqref{eq:finite-sample-bound}.
This is consistent with the proposition: the bound is an upper bound under a fourth-moment bound on the estimator alone, and Laplace noise has faster-than-polynomial tails so the actual decay can be sharper.
Recovery is exact in every seed by $n = 2000$.

\begin{figure}[t]
  \centering
  \includegraphics[width=\textwidth]{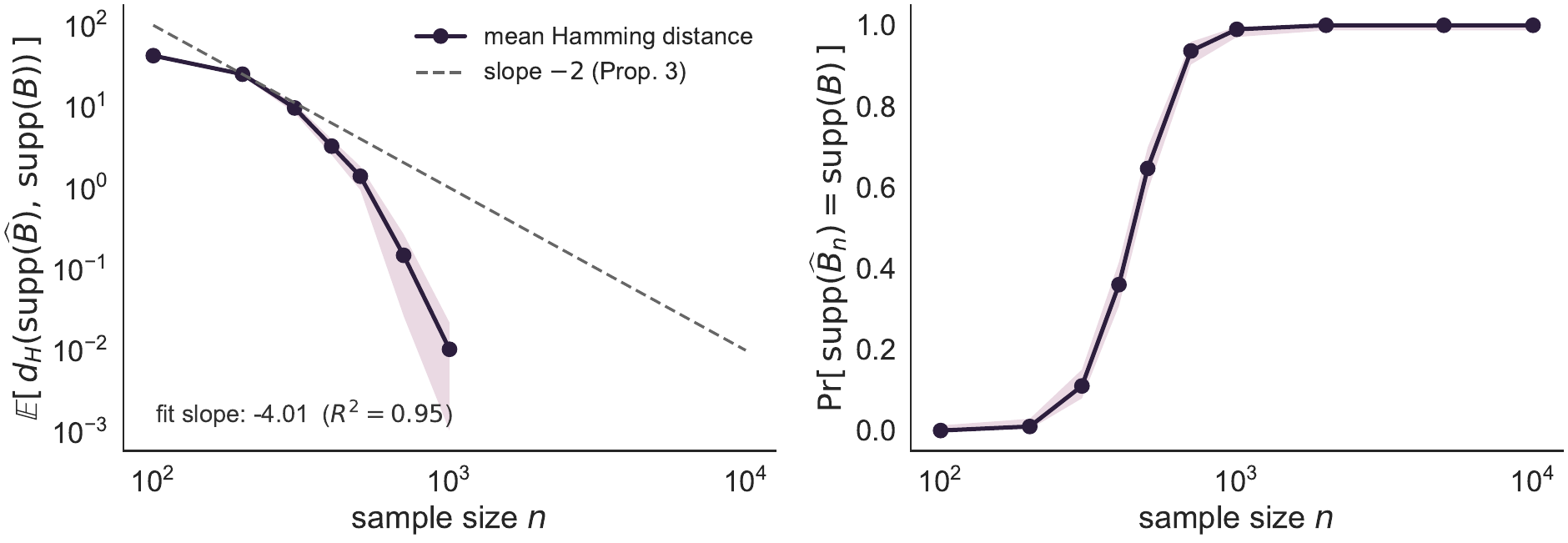}
  \caption{
    Empirical corroboration of \Cref{prop:finite-sample-rate} on the main-grid setup ($d = 10$, $\kappa = 4$, $\lambda = 0.5$, $\beta_{\min} \approx 0.51$, $\tau = \frac{\beta_{\min}}{2}$, 300 seeds per cell).
    \textbf{Left:} $\mathbb{E}[d_H(\operatorname{supp}(\widehat{B}),\, \operatorname{supp}(B))]$ vs $n$ on log-log; the dashed reference line has the $n^{-2}$ rate of \eqref{eq:finite-sample-bound} and the empirical slope is steeper, consistent with the bound being a conservative upper bound under Laplace noise.
    \textbf{Right:} $\Pr[\operatorname{supp}(\widehat{B}_n) = \operatorname{supp}(B)]$ vs $n$ with $95\%$ CI bands; the slope in the transition window $n \in [200, 1000]$ is $-2.9$, the log-log decay of the failure probability $1 - \Pr[\operatorname{supp}(\widehat{B}_n) = \operatorname{supp}(B)]$.
    Recovery is exact in every seed at $n \geq 2000$.
  }
  \label{fig:sample-complexity}
\end{figure}

\subsection{Proof of \Cref{prop:conservative}}
\label{app:failure}

\begin{proof}
  (a) Each non-trivial SCC contains at least one directed cycle.
  Removing any edge that lies on every cycle of an SCC disconnects the strong-connectivity relation on that SCC, partitioning it into smaller SCCs.
  (b) A spurious edge $\widehat{B}_{ij} \neq 0$ with $i \in \pi,\ j \in \pi'$ in the estimated graph, together with a path $\pi\leadsto\pi'$ already present in $G$, closes a directed cycle through the two clusters and merges them.
\end{proof}

\section{Additional Results}
\label{app:results}

\subsection{Comparison with related work}
\label{app:related-work-table}

\Cref{tab:comparison} positions our contribution among cyclic-identification and cluster-level causal works.
A dash (---) marks settings without a notion of clustering (variable-level methods); for GroupLiNGAM, unobserved confounders are permitted within blocks but not between them.

\begin{table}[t]
  \centering
  \caption{
    Comparison with related work.
    ``Clustering'' = whether the partition into clusters is assumed given or learned from data; ``---'' marks variable-level settings, which have no clustering notion.
    Our row assumes at most one Gaussian noise component.
    ``obs.''
    = observational; ``theory'' = theoretical results only; ``paired'' = measurements at both granularities for the same units.
  }
  \label{tab:comparison}
  \footnotesize
  \renewcommand{\arraystretch}{1.15}
  \setlength{\tabcolsep}{4pt}
  \begin{tabular}{@{}>{\raggedright\arraybackslash}p{2.35cm} >{\raggedright\arraybackslash}p{1.9cm} c c c >{\raggedright\arraybackslash}p{2.2cm}@{}}
    \toprule
    Work                              & Model class                       & Latent conf. &      Data      &          Clustering          & Output                                            \\
    \midrule
    This paper                        & cyclic LiNG SCM                   &      no      &      obs.      & \textbf{learned} (SCC floor) & SCC partition + condensation; C-DAG ID            \\
    \citet{lacerda2008discovering}    & cyclic LiNG SCM                   &      no      &      obs.      &             ---              & variable-level equivalence class                  \\
    \citet{pmlr-v286-drton25a}        & cyclic LiNG SCM (disjoint cycles) &      no      &      obs.      &             ---              & variable-level equivalence class                  \\
    \citet{dai2026distributional}     & cyclic LiNG SCM + latents         &     yes      &      obs.      &             ---              & variable-level equivalence class                  \\
    \citet{salehkaleybar2020learning} & acyclic LiNG SCM + latents        &     yes      &      obs.      &             ---              & causal order + possible causal effects            \\
    \citet{maeda2020rcd}              & acyclic LiNG SCM + latents        &     yes      &      obs.      &             ---              & causal structure + latent confounders             \\
    \citet{kawahara2010grouplingam}   & grouped linear non-Gaussian       &  in blocks   &      obs.      &           learned            & block ordering (+ condensation under sufficiency) \\
    \citet{massidda2024learning}      & linear acyclic SCM                &      no      &     paired     &           learned            & abstraction map + abstract SCM                    \\
    \citet{schooltink2025aligning}    & semi-Markovian SCMs (acyclic)     &     yes      &     theory     &            given             & C-DAG / functional-abstraction correspondence     \\
    \citet{yvernes2025relaxing}       & cyclic C-DAGs (over ADMGs)        &     yes      &     theory     &            given             & do-calculus over cluster-level queries            \\
    \citet{ferreira2025identifying}   & cyclic summary graphs (DMGs)      &     yes      &     theory     &            given             & macro effect identification                       \\
    \citet{anand2023causal}           & acyclic C-DAGs (over ADMGs)       &     yes      &     theory     &            given             & C-DAG effect identification                       \\
    \citet{anand2025discovering}      & Markovian DAGs                    &      no      &      obs.      &            given             & cluster-level equivalence class                   \\
    \citet{madaleno2026coarsening}    & acyclic DAG                       &      no      & interventional &           learned            & coarsening lattice                                \\
    \bottomrule
  \end{tabular}
\end{table}

\subsection{Strict disjoint-cycles micro experiment}
\label{app:disjoint-micro}

The main scalability sweep of \cref{sec:scalability} at density $\lambda = 0.5$ creates multiple overlapping cycles within each SCC of size $\ge 3$,
violating the disjoint-cycles assumption of~\cite{pmlr-v286-drton25a}.
To isolate the effect of that violation from the methodological gap,
we run a complementary comparison whose data-generating process satisfies the assumption \emph{exactly} by construction.

\paragraph{Setup}

We fix $d = 20$ with $\kappa = 5$ non-trivial SCCs,
inter-SCC density $\lambda = 0.5$, and \emph{zero} intra-SCC chord density:
each SCC is generated as a single directed Hamilton cycle and no further intra-SCC edges are added.
SCC sizes are random: each non-trivial SCC starts with two variables, each remaining node independently joins a uniformly chosen SCC with probability $1/2$, and the rest remain singletons.
Inter-SCC edges respect a random DAG ordering over the $\kappa = 5$ non-trivial SCCs and the singletons.
By construction, every directed simple cycle in the resulting graph lies inside a single SCC and is the unique cycle of that SCC,
so all simple cycles are pairwise vertex-disjoint, which is the structural premise of~\cite{pmlr-v286-drton25a}.
We sweep $n \in [10^2, 10^5]$ with $10$ seeds per cell,
in the stable weight regime ($w \in [0.5, 0.95]$, $\rho(B) < 1$),
and use skewed exponential noise so that the baseline's third-moment tests carry signal.
Both methods use the same fit configuration as the main scalability sweep.

\paragraph{Results}

\Cref{fig:disjoint-micro} reports ARI of the recovered partition, cluster-DAG $F_1$, and fit time.
Our method saturates at ARI $= 1$ and cluster-DAG $F_1 = 1$ from $n \ge 10^4$,
with fit time staying within $\sim 1$\,s across the full sweep.
The disjoint-cycles baseline plateaus at median ARI $\approx 0.13$ and
median cluster-DAG $F_1 \approx 0.60$ even at $n = 10^5$,
with fit time growing to $10^2$\,s.

Inspection of the baseline's predicted partitions explains the gap:
at $n \le 10^3$, every seed returns the trivial all-singletons partition
(the procedure's third-moment cycle tests at $\alpha = 0.01$ have insufficient power and confirm no cycle);
at $n \ge 10^4$, the output either contains seeds where some 2--3-cycles are correctly recovered and seeds where it merges most variables into a single cycle.

\begin{figure}[t]
  \centering
  \includegraphics[width=\textwidth]{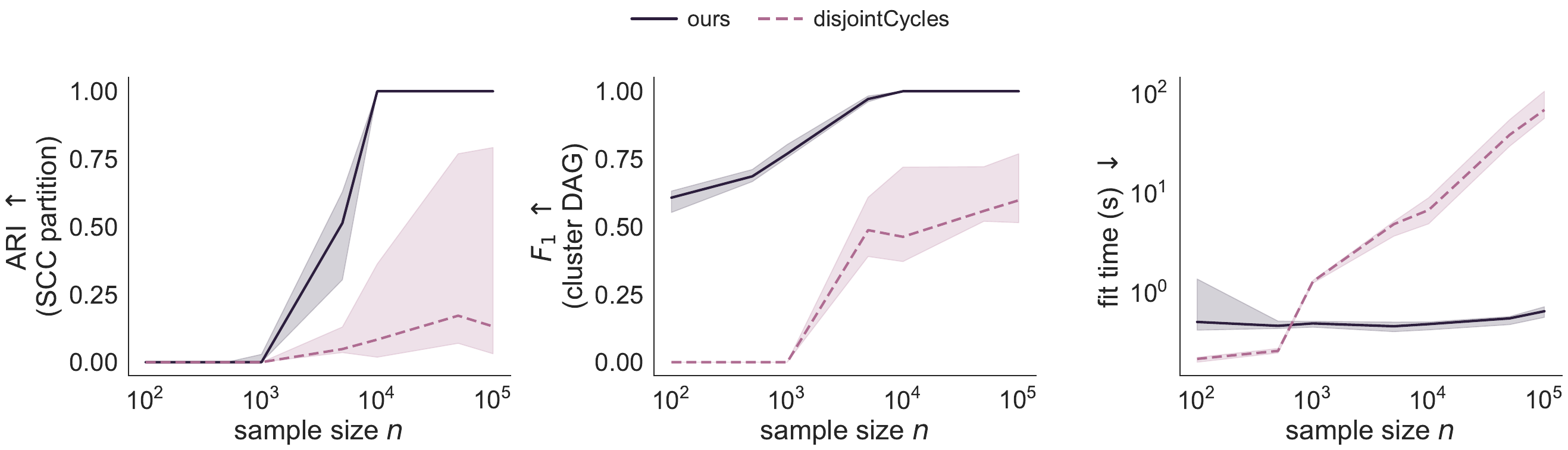}
  \caption{
    Strict disjoint-cycles micro experiment ($d = 20$,
    $\kappa = 5$, every SCC a single Hamilton cycle, $\lambda = 0.5$
    inter-SCC, $10$ seeds per cell).
    Solid: ours.
    Dashed:
    disjointCycles~\cite{pmlr-v286-drton25a}.
    Left to right: ARI of
    the recovered partition, cluster-DAG $F_1$, fit time (s, log).
    Shaded bands: $95\%$ CI across seeds.
  }
  \label{fig:disjoint-micro}
\end{figure}

\subsection{Comparison with GroupLiNGAM}
\label{app:grouplingam-exp}

GroupLiNGAM~\citep{kawahara2010grouplingam} estimates a partition of the variables into groups together with a causal order among the groups, without constraining the dependencies within a group.
Its output therefore has the same form as the condensation our method targets, which makes it a natural baseline for recovering the coarsening.
We compare the two methods on identical datasets across sample sizes.

\paragraph{Setup}

We use a single cell of the main synthetic grid: $d = 10$ with $\kappa = 4$ non-trivial SCCs, density $\lambda = 0.5$, and Laplace noise.
We consider both the stable regime ($w \in [0.5, 0.95]$, $\rho(B) < 1$) and the unstable regime (the same weights rescaled to $\rho(B) = 1.5$).
We sweep $n \in \{100, 500, 1000, 2000, 5000\}$ with $10$ seeds per cell.
GroupLiNGAM uses the implementation in the \texttt{lingam} package (version 1.13.0) with HSIC independence tests at $\alpha = 0.01$ and its own estimation of the between-group edges.
Our method uses threshold $\tau = 0.1$ and Hungarian candidate selection.
We report the ARI between the estimated groups and the true SCC partition, and the cluster-DAG $F_1$ obtained by projecting each method's estimated edges onto the true partition.

\paragraph{Results}

\Cref{fig:grouplingam-exp} reports ARI of the recovered partition, cluster-DAG $F_1$, and fit time.
Our method reaches median ARI $= 1$ and cluster-DAG $F_1 = 1$ from $n = 2000$ in the stable regime and from $n = 5000$ in the unstable regime, recovering the exact condensation in $10/10$ and $9/10$ seeds at $n = 5000$, with fit time below $1$\,s across the full sweep.
GroupLiNGAM attains a higher ARI at intermediate sample sizes ($n = 500$--$1000$), but plateaus at median ARI $\approx 0.84$ (stable) and $\approx 0.69$ (unstable) at $n = 5000$, recovering the exact condensation in only $3/10$ seeds per regime.
Its cluster-DAG $F_1$ stays close to ours ($1.00$ and $0.92$ at $n = 5000$), while its fit time grows from about $6$\,s at $n = 100$ to about $27$\,min at $n = 2000$ and about $11$\,h at $n = 5000$.

\begin{figure}[t]
  \centering
  \includegraphics[width=\textwidth]{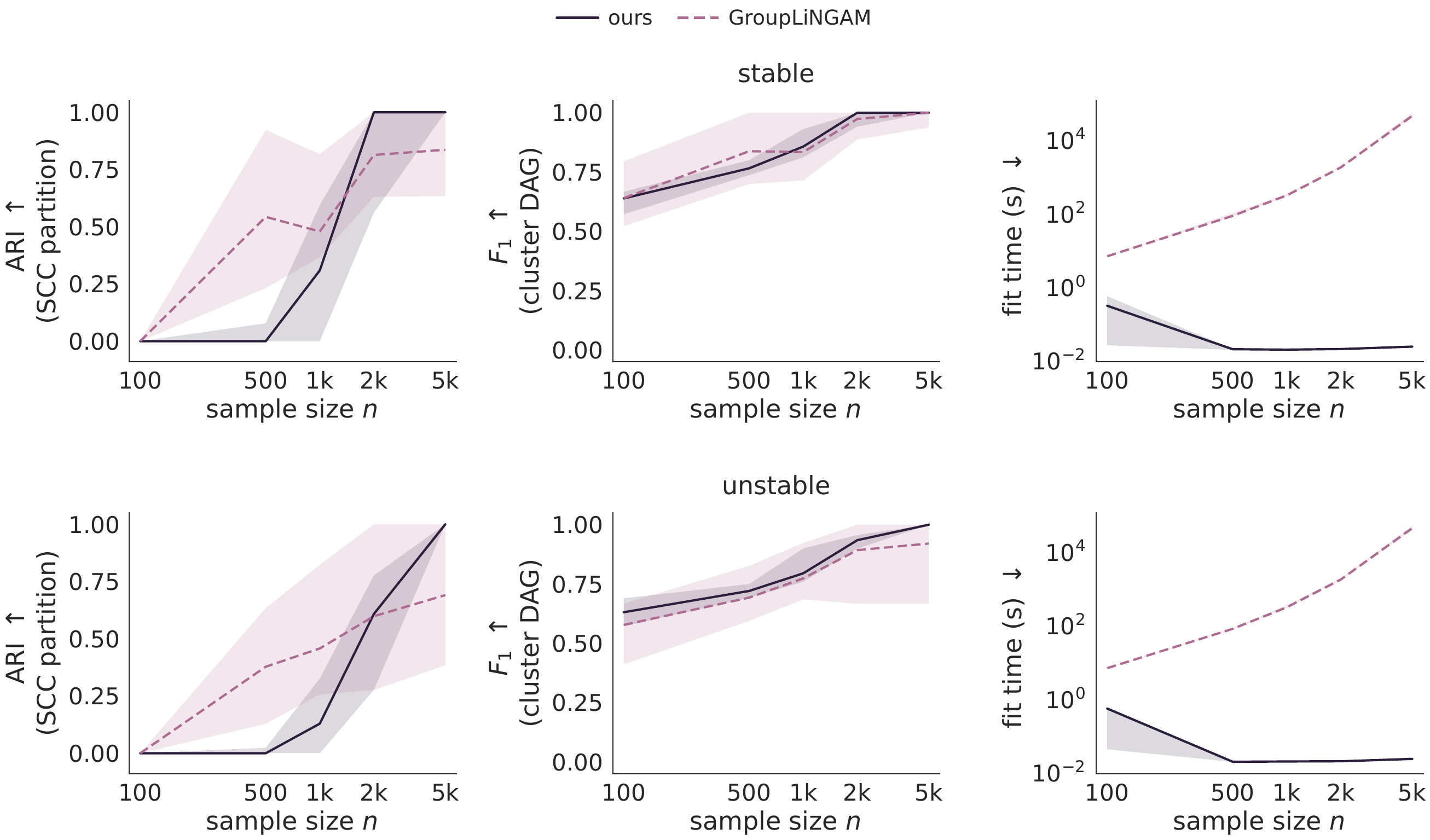}
  \caption{
    Comparison with GroupLiNGAM ($d = 10$, $\kappa = 4$, $\lambda = 0.5$,
    Laplace noise, $10$ seeds per cell).
    Top: stable regime.
    Bottom: unstable regime.
    Solid: ours.
    Dashed:
    GroupLiNGAM~\citep{kawahara2010grouplingam}.
    Left to right: ARI of
    the recovered partition, cluster-DAG $F_1$, fit time (s, log).
    Shaded bands: $95\%$ bootstrap CI of the median across seeds.
  }
  \label{fig:grouplingam-exp}
\end{figure}

\subsection{Threshold sensitivity and the two failure modes}
\label{app:threshold-sensitivity}

\Cref{alg:condensation} requires a threshold $\tau > 0$ on $|\widehat{B}_{ij}|$.
\Cref{prop:conservative} characterizes two opposite failure modes; this subsection probes both empirically, in two complementary ways: a threshold sweep with partition-level metrics (stable regime), and a pair-level classification of the complete main grid at the default threshold.
For the sweep we fix $d = 10$, $\kappa = 4$, $\lambda = 0.5$ in the stable regime ($\rho(B) < 1$) and vary $\tau$ across three orders of magnitude on a synthetic grid with 10 seeds per cell.
The number of non-trivial SCCs is exactly $\kappa = 4$ in every seed, but the singletons absorbed into non-trivial SCCs vary, so the true partition size $|\Pi|$ ranges over $\{4, 5, 6\}$ across seeds (median $5$).

\paragraph{Metrics}

\Cref{fig:threshold-sensitivity} reports three quantities.
The first two are the same metrics used in \Cref{sec:experiments}: ARI of the recovered SCC partition (left) and $F_1$ of the recovered cluster-DAG edges (middle).
The third is the size $|\widehat\Pi|$ of the estimated partition (right)---the failure modes:
\begin{itemize}
  \item $|\widehat\Pi| = 1$ (everything in one cluster): spurious entries above $\tau$ closed paths between distinct true SCCs and merged everything (Mode (b) of \Cref{prop:conservative}).
  \item $|\widehat\Pi| = d = 10$ (every variable is a singleton): true intra-SCC edges fell below $\tau$ and every SCC splintered (Mode (a)).
  \item $|\widehat\Pi| \in \{4, 5, 6\}$ (the $\tau$-independent range of the truth): the recovered partition agrees with the per-seed truth.
\end{itemize}

\paragraph{Pair-level error accounting}

Aggregate ARI and $F_1$ show that recovery improves with $n$, but not \emph{how} it fails.
We therefore also measure split and merge errors directly: a \emph{split pair} places two variables of the same true SCC in different estimated clusters, a \emph{merge pair} places two variables of different true SCCs in the same estimated cluster; each rate is normalized by its eligible pair count, and every run is classified as \emph{exact}, \emph{split-only}, \emph{merge-only}, or \emph{mixed}.
At the default $\tau = 0.1$ on the complete main grid ($d = 10$, both regimes, $\kappa \in \{3,4,5\}$, $\lambda \in \{0.3,0.5,0.8\}$, $n \in \{50, \dots, 5000\}$, ten seeds; 900 runs), every observed error is a merge: all runs with $n \leq 500$ are merge-only, and exact recovery reaches $98.9\%$ (stable) and $81.1\%$ (unstable) at $n = 5000$.
Read in pair-level terms, the sweep below tells the same story as the partition sizes: merges dominate at $\tau \leq 0.01$ (median merge rate $1.0$ at every $n$ tested), an error-free interval appears at intermediate thresholds and widens with $n$, and splits appear only at $\tau \geq 0.7$ (median split rate $0.63$--$1.0$).

\paragraph{Results}

At $n \geq 5000$ the safe band is wide and roughly $\tau \in [0.1, 0.5]$: ARI $= 1$ and $F_1 = 1$ in every seed and $|\widehat\Pi|$ matches the per-seed truth.
Outside that band the figure shows the two failure modes unambiguously: at $\tau \leq 0.01$ the right panel collapses to $|\widehat\Pi| = 1$, at $\tau = 1$ it climbs to $|\widehat\Pi| = d = 10$, and ARI / $F_1$ collapse correspondingly.
The plateau begins later as $n$ shrinks (at $n = 500$ it begins only at $\tau \geq 0.2$), reflecting the larger finite-sample fluctuations of $\widehat{B}$.
The default $\tau = 0.1$ used throughout \Cref{sec:experiments} sits well inside the safe band whenever $n \geq 5000$.\footnote{The plateau at $F_1 \approx 2/3$ for $\tau \to 0$: At very small $\tau$ the recovered $\widehat{B}$ is near-fully dense, so the predictor effectively guesses every directed cluster-pair as an edge: recall is $1$ and precision is the fraction of cluster-pairs that are truly connected, which for $\lambda = 0.5$ inter-cluster density gives precision $\approx 1/2$ and hence $F_1 \approx 2/3$.
The plateau at $1$ for moderate $\tau$ is the regime of true recovery.}

\begin{figure}[t]
  \centering
  \includegraphics[width=\textwidth]{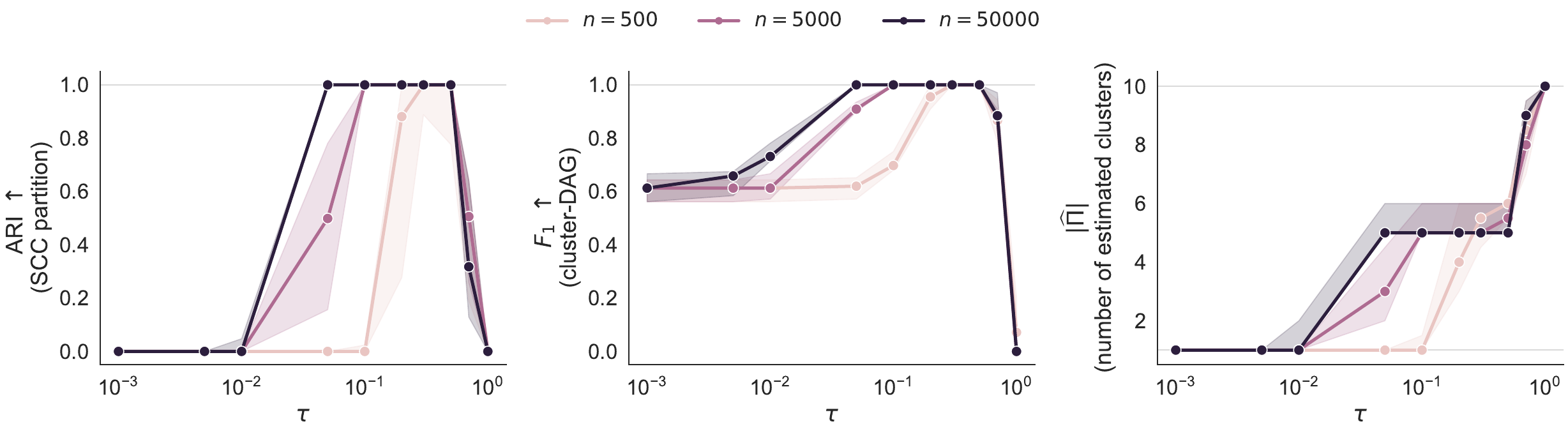}
  \caption{
    Threshold sensitivity at $d = 10$, $\kappa = 4$, $\lambda = 0.5$ (stable regime, 10 seeds).
    \textbf{Left:} ARI of the recovered SCC partition.
    \textbf{Middle:} $F_1$ of the recovered cluster-DAG edges.
    \textbf{Right:} number of estimated clusters $|\widehat\Pi|$; the true $|\Pi|$ varies by seed across $\{4, 5, 6\}$ with median $5$, and inside the safe band $|\widehat\Pi|$ matches the per-seed truth.
    Curves are colored by sample size.
    The right panel makes the two failure modes of \Cref{prop:conservative} explicit: $\tau \to 0$ collapses to $|\widehat\Pi| = 1$ (Mode (b)) and $\tau \to 1$ saturates at $|\widehat\Pi| = d = 10$ (Mode (a)).
    Shaded bands: $95\%$ CI across seeds.
  }
  \label{fig:threshold-sensitivity}
\end{figure}

\subsection{Candidate-selection ablation}
\label{app:ablation}

\Cref{cor:condensation-id} selects a single admissible permutation with the Hungarian algorithm, whereas the original recipe of~\cite{lacerda2008discovering} enumerates all admissible candidates and applies the first-stable filter (\cref{sec:approach-oica}).
To isolate the effect of the selection rule, we compare four procedures on the same per-dataset FastICA estimate $\widehat W$: (i) full enumeration with first-stable selection; (ii) full enumeration with one admissible candidate drawn uniformly at random from the enumerated list, repeated five times per dataset and averaged, which tests sensitivity to the representative choice; (iii) Hungarian selection as in \Cref{alg:condensation}; and (iv) a uniformly random row permutation \emph{without} enforcing admissibility, repeated five times, as a negative control.
The setting uses $d = 20$, $\kappa = 5$ non-trivial SCCs, and $30$ seeds per cell, with $n \in \{10^4, 5 \cdot 10^4\}$.
ICA and selection times are reported separately.

\begin{table}[t]
  \centering
  \caption{
    Candidate-selection ablation at $d = 20$ with $\kappa = 5$
    non-trivial SCCs ($30$ seeds per cell).
    Exact denotes exact recovery of both the SCC partition and the condensation.
    ARI and exact recovery are averaged across seeds.
    Times are medians.
  }
  \label{tab:ablation}
  \small
  \setlength{\tabcolsep}{4.5pt}
  \resizebox{\linewidth}{!}{%
    \begin{tabular}{llcccccccc}
      \toprule
               &                         &     \multicolumn{2}{c}{ARI}     & \multicolumn{2}{c}{Exact (\%)}  & \multicolumn{2}{c}{ICA time (s)} &   \multicolumn{2}{c}{Selection time}   \\
      \cmidrule(lr){3-4} \cmidrule(lr){5-6} \cmidrule(lr){7-8} \cmidrule(lr){9-10}
      Regime   & Selection               & $n{=}10^4$ & $n{=}5{\cdot}10^4$ & $n{=}10^4$ & $n{=}5{\cdot}10^4$ & $n{=}10^4$ & $n{=}5{\cdot}10^4$ &     $n{=}10^4$     & $n{=}5{\cdot}10^4$ \\
      \midrule
      Stable   & enum.\ + first-stable   &   0.968    &       1.000        &     90     &        100         &   0.015    &       0.059        &      6.054\,s      &      6.083\,s      \\
      Stable   & enum.\ + uniform random &   0.741    &       0.997        &     46     &         98         &   0.015    &       0.059        &      6.054\,s      &      6.083\,s      \\
      Stable   & ours                    &   0.992    &       1.000        &    93.3    &        100         &   0.015    &       0.059        & \textbf{0.044\,ms} & \textbf{0.036\,ms} \\
      Stable   & arbitrary control       &   0.001    &       0.002        &     0      &         0          &   0.015    &       0.059        &     0.031\,ms      &     0.030\,ms      \\
      \midrule
      Unstable & enum.\ + first-stable   &   0.868    &       0.994        &    51.7    &        96.6        &   0.015    &       0.062        &      6.379\,s      &      6.070\,s      \\
      Unstable & enum.\ + uniform random &   0.855    &       0.998        &    53.1    &        98.6        &   0.015    &       0.062        &      6.379\,s      &      6.070\,s      \\
      Unstable & ours                    &   0.935    &       1.000        &    76.7    &        100         &   0.015    &       0.062        & \textbf{0.041\,ms} & \textbf{0.041\,ms} \\
      Unstable & arbitrary control       &   0.001    &       0.001        &     0      &         0          &   0.015    &       0.062        &     0.030\,ms      &     0.031\,ms      \\
      \bottomrule
    \end{tabular}
  }
\end{table}

\Cref{tab:ablation} shows four things.
First, at the larger sample size, condensation recovery is largely insensitive to \emph{which} admissible representative is selected: uniform-random and first-stable selection both achieve 96.6--100\% exact recovery at $n = 5 \cdot 10^4$---but uniform-random selection still pays the same enumeration cost, so it saves no computation.
Second, Hungarian selection is at least as accurate as first-stable throughout, with median selection times of $0.036$--$0.044\,$ms versus $6.05$--$6.38\,$s for enumeration, which is the computational advantage claimed in \Cref{thm:time-complexity}.
Third, none of the control draws is admissible and none recovers the condensation: if $m$ permutations are admissible, a uniform draw is admissible with probability $m/d!$, highlighting the importance of enforcing admissibility.
Fourth, in the unstable regime we rescale $B$ to spectral radius $\rho(B) = 1.5 > 1$, deliberately violating the stability assumption of first-stable selection.
The selected adjacency matrix does not need to coincide with the generating matrix \(B\), but its condensation is unchanged.---yet Hungarian selection (ours) still reaches $100\%$ exact recovery at $n = 5 \cdot 10^4$: the invariance of \Cref{thm:condensation-id} is what exact recovery rests on.

\subsection{Whole-SCC hard-intervention experiment}
\label{app:intervention-expt}

We evaluate the pipeline of \cref{sec:effects} end to end: recover $\widehat{G}^{\mathrm{sc}}$ from observational data with \Cref{alg:condensation}, then estimate downstream interventional effects on it.
We generate 40 linear non-Gaussian cyclic SCMs with $d = 10$, four non-trivial SCCs, and density $0.5$ (20 per regime, stable and unstable as in \cref{sec:experiments}), each evaluated at $n \in \{100, 500, 1000, 5000, 5{\cdot}10^4\}$.
For each SCM we perform a hard intervention that sets every variable $X_j$ of its largest non-trivial SCC with a downstream SCC to $\hat\mu_j + \hat\sigma_j$, and take the exact mean change of each downstream SCC, computed from the generating SEM, as the target effect.

\paragraph{Estimation}

We treat the recovered condensation as the C-DAG.
For a hard intervention setting a whole target SCC $S$ to $c$, the cluster-level $g$-formula identifies the post-intervention distribution as
\[
  p(x_{V \setminus S} \mid \mathrm{do}(X_S = c)) \; = \; \prod_{\pi \in \Pi \setminus S} p(x_\pi \mid x_{\mathrm{pa}(\pi)}),
\]
with $X_S$ fixed to $c$.
In our linear setting we estimate each conditional mean by multivariate least squares and propagate the resulting means through the condensation in topological order; the procedure requires no orientation \emph{within} any SCC.

\paragraph{Metric}

For each downstream SCC we define the normalized effect error
\[
  \frac{|\text{estimated effect} - \text{true effect}|}{\max\bigl( |\text{true effect}|,\; 0.1 \times \mathrm{SD}(\text{SCC mean}) \bigr)},
\]
where the lower bound in the denominator prevents near-zero effects from producing unstable ratios.
We take the median across downstream SCCs within each dataset, then across the 20 SCMs in each regime.

\begin{table}[t]
  \centering
  \caption{
    Median normalized effect error for whole-SCC hard interventions, computed on the \emph{true} condensation.
    At $n = 5{\cdot}10^4$, \Cref{alg:condensation} recovers the complete SCC partition and condensation in every tested SCM, so the end-to-end procedure attains the same errors.
  }
  \label{tab:intervention}
  \begin{tabular}{lcc}
    \toprule
    Regime   & $n = 100$ & $n = 5{\cdot}10^4$ \\
    \midrule
    Stable   &   0.289   &       0.013        \\
    Unstable &   0.219   &       0.008        \\
    \bottomrule
  \end{tabular}
\end{table}

With the true condensation, the median normalized effect error falls from $0.289$ to $0.013$ in the stable regime and from $0.219$ to $0.008$ in the unstable regime as $n$ grows from $100$ to $5{\cdot}10^4$ (\Cref{tab:intervention}).
At $n = 5{\cdot}10^4$, \Cref{alg:condensation} recovers the complete SCC partition and condensation in every tested SCM, so the end-to-end procedure estimates the downstream effects with the same median errors.
At $n = 5000$, structural recovery succeeds in $1.00$ (stable) and $0.75$ (unstable) of SCMs; whenever recovery is exact, the effect estimate equals the true-structure estimate, because both use the same regressions.
We never map incorrectly estimated clusters onto the true ones.
At intermediate sample sizes, where recovery is not always exact, we do not evaluate the propagation of finite-sample structural errors into the downstream effect estimates.
This experiment covers whole-SCC hard interventions; we make no claim about single-variable interventions or other within-SCC quantities.

\newpage
\section*{NeurIPS Paper Checklist}

\begin{enumerate}

\item {\bf Claims}
    \item[] Question: Do the main claims made in the abstract and introduction accurately reflect the paper's contributions and scope?
    \item[] Answer: \answerYes{} 
    \item[] Justification: The abstract and introduction accurately state the main claims regarding condensation identifiability and the strongly connected components (SCC) floor. This directly matches the theoretical contributions detailed in Section 4.

\item {\bf Limitations}
    \item[] Question: Does the paper discuss the limitations of the work performed by the authors?
    \item[] Answer: \answerYes{} 
    \item[] Justification: The paper explicitly notes the limitation of requiring linear non-Gaussian noise and no latent confounders. It also discusses open directions to address these limitations, such as latent-confounder generalisations.

\item {\bf Theory assumptions and proofs}
    \item[] Question: For each theoretical result, does the paper provide the full set of assumptions and a complete (and correct) proof?
    \item[] Answer: \answerYes{} 
    \item[] Justification: Assumptions are stated in each result and in \cref{sec:cyclic-scm,sec:approach-oica}, plus \Cref{assum:candidate-consistency} for the finite-sample results. The main text gives proof sketches. Complete proofs of \Cref{prop:scc-floor}, \Cref{thm:condensation-id} (via \Cref{lem:perm-condensation}), \Cref{lem:cond-faith}, \Cref{thm:time-complexity}, \Cref{prop:conservative}, \Cref{prop:finite-sample-rate} and \Cref{cor:sample-complexity} are in \Cref{app:proofs}.

    \item {\bf Experimental result reproducibility}
    \item[] Question: Does the paper fully disclose all the information needed to reproduce the main experimental results of the paper to the extent that it affects the main claims and/or conclusions of the paper (regardless of whether the code and data are provided or not)?
    \item[] Answer: \answerYes{} 
    \item[] Justification: The experimental setup thoroughly details the data-generating process, including the number of variables, SCCs, edge densities, edge weights, and noise distributions. It also clearly outlines the two spectral-radius regimes tested and the specific ICA-LiNG-D method used.
    All results in this paper are synthetic; the data-generating process is fully specified in \cref{sec:experiments} (graph construction, noise distributions, sample sizes, and seeds), and each figure states its own configuration.
All theoretical claims are proved in \Cref{app:proofs}, including the complete assumptions, the finite-sample analysis of \Cref{prop:finite-sample-rate}, and an empirical corroboration of its rate.
Algorithm~1 requires only observational data and a threshold choice; \Cref{prop:conservative} characterizes the two failure modes of thresholding, and \Cref{app:threshold-sensitivity} quantifies their effect empirically.
The open-source implementation (footnote~\ref{fn:code}) reproduces all experimental figures (\Cref{fig:results,fig:sample-complexity,fig:scalability,fig:disjoint-micro,fig:grouplingam-exp,fig:threshold-sensitivity}) as well as the additional experiments of \Cref{app:ablation}, \Cref{app:intervention-expt}, and the split/merge diagnostics of \Cref{app:threshold-sensitivity} from raw data through a Snakemake workflow with one target per figure, and includes the environment specification and unit tests of the core routines.
Hardware and fit times are reported in \cref{sec:experiments}.

\item {\bf Open access to data and code}
    \item[] Question: Does the paper provide open access to the data and code, with sufficient instructions to faithfully reproduce the main experimental results, as described in supplemental material?
    \item[] Answer: \answerYes{} 
    \item[] Justification: Code and scripts reproducing every experiment are public at \url{https://github.com/fmfsa/coarsening-cycles} (footnote~\ref{fn:code}). They include a Snakemake workflow with one target per figure and the environment specification. All data are synthetic and generated by these scripts.

\item {\bf Experimental setting/details}
    \item[] Question: Does the paper specify all the training and test details (e.g., data splits, hyperparameters, how they were chosen, type of optimizer) necessary to understand the results?
    \item[] Answer: \answerYes{} 
    \item[] Justification: The paper explicitly specifies necessary details, such as sample sizes ranging from $50$ to $10^5$, replication across $10$ random seeds, and specific execution parameters like thresholding values of $\tau=0.1$.

\item {\bf Experiment statistical significance}
    \item[] Question: Does the paper report error bars suitably and correctly defined or other appropriate information about the statistical significance of the experiments?
    \item[] Answer: \answerYes{} 
    \item[] Justification: \Cref{fig:results,fig:scalability,fig:disjoint-micro,fig:threshold-sensitivity} show 95\% confidence bands across random seeds. \Cref{fig:grouplingam-exp} shows 95\% bootstrap CIs of the median across seeds. Each caption states the number of seeds and the type of band. \Cref{tab:ablation,tab:intervention} report means or medians over seeds, without intervals.

\item {\bf Experiments compute resources}
    \item[] Question: For each experiment, does the paper provide sufficient information on the computer resources (type of compute workers, memory, time of execution) needed to reproduce the experiments?
    \item[] Answer: \answerYes{} 
    \item[] Justification: \cref{sec:experiments} states that all experiments ran on a single CPU workstation (AMD Ryzen 9 9950X, 128 GB RAM) without GPUs. Per-method fit times are reported in \Cref{fig:scalability,fig:disjoint-micro,fig:grouplingam-exp} and \Cref{tab:ablation}.
    
\item {\bf Code of ethics}
    \item[] Question: Does the research conducted in the paper conform, in every respect, with the NeurIPS Code of Ethics \url{https://neurips.cc/public/EthicsGuidelines}?
    \item[] Answer: \answerYes{} 
    \item[] Justification: The research methodology is purely theoretical and algorithmic, leveraging synthetically generated data, thereby meaning it natively conforms to the Code of Ethics without any obvious violations.

\item {\bf Broader impacts}
    \item[] Question: Does the paper discuss both potential positive societal impacts and negative societal impacts of the work performed?
    \item[] Answer: \answerNA{} 
    \item[] Justification: This work is methodological: it identifies cluster-level causal structure (the SCC partition and condensation) from observational data in linear non-Gaussian models with cycles. We see no risks specific to this work beyond the standard caveats of linear non-Gaussian causal discovery: outputs are only as reliable as the modelling assumptions, and a recovered cluster-level edge does not pinpoint which intra-cluster variables drive the relationship. The methods do not enable applications inaccessible to existing acyclic causal-discovery tools.
    
\item {\bf Safeguards}
    \item[] Question: Does the paper describe safeguards that have been put in place for responsible release of data or models that have a high risk for misuse (e.g., pre-trained language models, image generators, or scraped datasets)?
    \item[] Answer: \answerNA{} 
    \item[] Justification: The paper's experiments strictly use synthetically generated math datasets representing linear cyclic SCMs and inherently pose no high risk for misuse.

\item {\bf Licenses for existing assets}
    \item[] Question: Are the creators or original owners of assets (e.g., code, data, models), used in the paper, properly credited and are the license and terms of use explicitly mentioned and properly respected?
    \item[] Answer: \answerNA{} 
    \item[] Justification: The paper does not report using existing real-world external datasets or restricted models, as experiments are executed using generated synthetic data matrices.

\item {\bf New assets}
    \item[] Question: Are new assets introduced in the paper well documented and is the documentation provided alongside the assets?
    \item[] Answer: \answerYes{} 
    \item[] Justification: The codebase repository is provided and well documented. 

\item {\bf Crowdsourcing and research with human subjects}
    \item[] Question: For crowdsourcing experiments and research with human subjects, does the paper include the full text of instructions given to participants and screenshots, if applicable, as well as details about compensation (if any)? 
    \item[] Answer: \answerNA{} 
    \item[] Justification: The entire scope of the research focuses on causality, graph theory and synthetic data, meaning no crowdsourcing or human subject trials were conducted.

\item {\bf Institutional review board (IRB) approvals or equivalent for research with human subjects}
    \item[] Question: Does the paper describe potential risks incurred by study participants, whether such risks were disclosed to the subjects, and whether Institutional Review Board (IRB) approvals (or an equivalent approval/review based on the requirements of your country or institution) were obtained?
    \item[] Answer: \answerNA{} 
    \item[] Justification: There are no study participants or human subjects involved in this theoretical algorithm paper.

\item {\bf Declaration of LLM usage}
    \item[] Question: Does the paper describe the usage of LLMs if it is an important, original, or non-standard component of the core methods in this research? Note that if the LLM is used only for writing, editing, or formatting purposes and does \emph{not} impact the core methodology, scientific rigor, or originality of the research, declaration is not required.
    \item[] Answer: \answerYes{} 
    \item[] Justification: A full AI use statement is provided.

\end{enumerate}

\end{document}